\pdfoutput=1
\documentclass{article}
\usepackage{iclr2020_conference_noheader, times}
\usepackage{natbib}
\usepackage{amsmath,amsthm,amssymb,amsfonts}
\usepackage{tikz}
\usetikzlibrary{bayesnet}
\usepackage{algorithm}
\usepackage{graphicx}
\usepackage{booktabs}
\usepackage{multirow}

\newtheorem{theorem}{Theorem}

\DeclareMathOperator{\vect}{vec}
\DeclareMathOperator{\tr}{tr}
\DeclareMathOperator{\kron}{\otimes}

\DeclareMathOperator{\Cat}{Cat}
\DeclareMathOperator{\Diag}{diag}
\DeclareMathOperator{\DiagPart}{diag\_part}
\DeclareMathOperator{\Softmax}{softmax}
\DeclareMathOperator{\Cov}{Cov}
\DeclareMathOperator{\chol}{chol}

\title{Meta-Learning with Variational Bayes}
\author{Lucas D. Lingle \\
\texttt{lucasdaxlingle@gmail.com}
}

\begin{document}
\maketitle
\begin{abstract}
%
The field of meta-learning seeks to improve the ability of today's machine learning systems to adapt efficiently to small amounts of data. Typically this is accomplished by training a system with a parametrized update rule to improve a task-relevant objective based on supervision or a reward function. However, in many domains of practical interest, task data is unlabeled, or reward functions are unavailable. In this paper we introduce a new approach to address the more general problem of \emph{generative meta-learning}, which we argue is an important prerequisite for obtaining human-level cognitive flexibility in artificial agents, and can benefit many practical applications along the way. Our contribution leverages the AEVB framework and mean-field variational Bayes, and creates fast-adapting latent-space generative models. At the heart of our contribution is a new result, showing that for a broad class of deep generative latent variable models, the relevant VB updates do not depend on any generative neural network. The theoretical merits of our approach are reflected in empirical experiments.
%
\end{abstract}  
\section{Introduction}

The ability to adapt quickly is a key aspect of human intelligence, and in recent years the field of \emph{meta-learning}, or learning to learn efficiently by leveraging prior experience, has become an active topic of machine learning research. In particular, \emph{optimization-based} meta-learning has emerged as a strong contender for creating fast-adapting learning systems \citep{Finn2017}. While most work on optimization-based meta-learning has focused on the supervised and reinforcement learning settings, in this work we argue in favor of \emph{unsupervised meta-learning}--that is, the ability to learn from small amounts of \emph{unlabeled} data. We argue that this ability has generally been overlooked by prior art, and is a promising direction of research, due to its potential to allow reinforcement learning agents to adapt and plan in unfamiliar domains \citep{Nagabandi2019}, to address the credit assignment problem more robustly than may be possible with model-free RL \citep{Harutyunyan2019}, and to operate in meta-learning regimes where meeting the outer-loop data requirements of black-box adaptation methods is infeasible \citep{Mishra2018}.  

The specific focus of this work is on unsupervised meta-learning of a \emph{generative} variety--the goal is to create fast-adapting generative models that adapt to small amounts of data, without requiring that this data be stored in raw form. Our work succeeds in meeting this goal, and we introduce a novel approach for creating \emph{highly expressive and scalable} latent variable models with fast adaptation capabilities. Our approach is based on mean-field variational Bayes \citep{WainwrightJordan2008}, and is made possible by a new result showing that the VB updates for `non-perceptual' variables--controlling the latent space portion of the generative model--do not depend on the decoder neural network, and thus can be made tractable by a thoughtful design of the latent space distribution. Unlike MAML, which uses gradient-based update rules for inner-loop adaptivity, mean-field variational Bayes follows the natural gradient and thus can be understood as a highly efficient second order method \citep{Hoffman2013}. Moreover unlike a gradient-based update rule, mean-field variational Bayes can take very large steps while maintaining a monotonic improvement guarantee on the evidence lower bound. 

Our approach to inner-loop updates is combined with gradient-based outer-loop training, which can be based on an objective measuring generalization \citep{Finn2017, Garnelo2018}, or in a setting where only adaptivity is measured, in which case the models can be interpreted as a form of memory \citep{Edwards2017, Wu2018a, Wu2018b}. The models developed under our approach train stably on an unbiased estimator of the evidence lower bound, are computationally lightweight, are easy to benchmark, can organize information dynamically, and they maintain an expressive distribution over latent space, free of the common failure modes of amortized inference, such as posterior collapse and latent variable collapse. In the memory setting, our models are shown to generate observations that are crisp, coherent, and strikingly similar to the observations in test-set episodes. Importantly, our contribution succeeds in showing that highly scalable and expressive latent-space generative models can be created, efficiently trained, and efficiently updated, using this combination of classical and modern techniques. Also, one of our models is resizable. 

Unlike the model from \citet{Edwards2017}, our models address the feature binding problem and do not use amortized inference for a memory variable. Unlike the models by \citet{Wu2018a, Wu2018b}, our models define a flexible multimodal distribution over latent space. Unlike the model from \citet{Rao2019}, our models adapt quickly and have no parametric dependency on memory size. Unlike the model by \citet{Marblestone2020}, ours can be trivially resized at test time and has a well-defined globally normalized distribution. Unlike the models by \citet{Johnson2016} and \citet{WLin2018}, our models do not require a variational message passing interface, we optimize the evidence lower bound during inference, we demonstrate the efficacy of our approach on natural images, and we focus on the meta-learned episodic setting which is arguably the main reason to use variational Bayes to begin with. This paper is organized as follows: in Sections 2 we present our method; in Section 3 we discuss variations on the basic approach; in Section 4 we present experiments; in Section 5 we conclude.  
\section{Meta-Learning with Variational Bayes}
\label{vbm_specification}

\subsection{Generative Model}

In this work, we consider generative models for grouped visual data (`episodes') $\mathbf{X} = \{\mathbf{x}_{t}\}_{t=1}^{T}$ of the form 
\begin{align}
p(\mathbf{X}, \mathbf{Z}, \mathbf{Y}, \boldsymbol{\Omega}) & = p(\boldsymbol{\Omega})\prod_{t=1}^{T}p(\mathbf{y}_{t}|\boldsymbol{\Omega})p(\mathbf{z}_{t}|\mathbf{y}_{t},\boldsymbol{\Omega})p(\mathbf{x}_{t}|\mathbf{z}_{t})
\end{align}
where $\mathbf{Z}$, $\mathbf{Y}$ are each shorthand for the joint collections of local variables $\{\mathbf{z}_{t}\}_{t=1}^{T}$ and $\{\mathbf{y}_{t}\}_{t=1}^{T}$. Each $\mathbf{z}_{t}$ is a \emph{perceptual code} decoded by a neural network $p(\mathbf{x}_{t}|\mathbf{z}_{t})$. Here, the notation $\mathbf{y}_{t}$ denotes the collection of all other local latent variables at timestep $t$, and $\boldsymbol{\Omega}$ denotes the collection of all episode-level latent variables. 

In particular, we consider generative models defined so that the local variables $\mathbf{x}_{t}, \mathbf{z}_{t}, \mathbf{y}_{t}$ at the various timesteps are conditionally independent and identically distributed, given the top-level variables $\boldsymbol{\Omega}$. Models of this kind are known as conditionally independent hierarchical models \citep{Kass1989}. 

\begin{figure}[ht]
   \begin{center}
    \begin{tabular}{cccc}
      \begin{tikzpicture}
  \node[obs]                               (x) {$x_{t}$};
  \node[latent, above=of x] (z) {$z_{t}$};
  \node[latent, above=of x, xshift=1.5cm]  (y) {$y_{t}$};
  \node[latent, above=of z]            (Omega) {$\Omega$};

  \edge {Omega, y} {z};
  \edge {Omega} {y};
  \edge {z} {x}; 

  \plate {episode} {(x)(z)(y)} {$T$};
  \plate {} {(x)(y)(z)(Omega)(episode.north west)(episode.south east)} {$N$};

\end{tikzpicture}
    \end{tabular}
  \end{center}
\caption{The generative models considered, described with plate notation.}
\label{generative_model_plate_notation}
\end{figure}
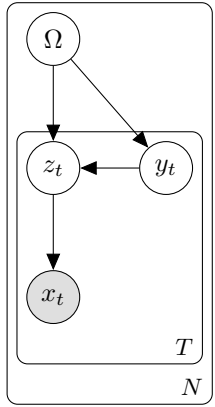

In this paper, we consider specific models under this broad structure, motivated by the analytic tractability of variational Bayesian inference in the specific models chosen. Not all models in this class possess such tractability. The generality of the presentation above serves only to provide a unified treatment of the proposed mean-field variational Bayesian approach. The detailed requirements for its applicability are given in the next section. 

\subsection{Inference Model}

As is typical of traditional mean-field variational Bayesian inference, we construct an inference model with some of the stochastic dependencies between the latent variables removed. In particular, for an episode $\mathbf{X} = \{\mathbf{x}_{t}\}_{t=1}^{T}$, we consider  inference models of the form
\begin{align}
q(\mathbf{Z}, \mathbf{Y}, \boldsymbol{\Omega}) & = q(\mathbf{Z})q(\mathbf{Y})q(\boldsymbol{\Omega}).
\end{align}
As discussed in Appendix A, 
the variational distributions $q(\mathbf{Z})$ and $q(\mathbf{Y})$ would reduce to distributions of the form $\prod_{t=1}^{T}q(\mathbf{z}_{t})$ and $\prod_{t=1}^{T}q(\mathbf{y}_{t})$, as a consequence of optimization, rather than due to additional assumptions. 

Due to the intractabilities associated with computing each optimal $q(\mathbf{z}_{t})$ in closed-form, we adopt an approach for $q(\mathbf{Z})$ based on fixed-form variational inference \citep{Salimans2012}, of the amortized variety \citep{KingmaWelling2013}. We structure each $q(\mathbf{z}_{t})$ as a multivariate Gaussian distribution with diagonal covariance, and we use a neural network to select the variational parameters of each $q(\mathbf{z}_{t})$. Motivated by the fact that the optimal form of the joint distribution $q(\mathbf{Z})$ does not possess stochastic dependencies between the $\mathbf{z}_{t}$'s, we structure the joint distribution $q(\mathbf{Z})$ of our fixed-form model according to this same specification. 


Subsequent to selecting the variational parameters for each $q(\mathbf{z}_{t})$, we use a coordinate ascent variational inference algorithm, based on mean-field variational Bayes or structured mean-field variational Bayes, to iteratively update the other variational distributions. Coordinate ascent variational inference is applicable to any latent variables with priors and complete conditionals in the exponential family, and which satisfy conditional conjugacy. The updates can be derived from first principles, by applying variational calculus to free-form optimization of the evidence lower bound. Under the aforementioned conditions, the updates are analytically tractable, closed under iterations, and guaranteed to monotonically improve the evidence lower bound \citep{Jordan1999, Beal2003}. 

A short derivation in Appendix B 
shows that the coordinate ascent variational inference updates do not depend on the generative neural network. The variational Bayesian update formulas for several specific models are provided in Appendix C. 

Note that after one iteration of this type of algorithm, the inference model has the form: 
\begin{align}
q(\mathbf{Z}, \mathbf{Y}, \boldsymbol{\Omega}) &= q(\boldsymbol{\Omega})\prod_{t=1}^{T}q(\mathbf{y}_{t})q(\mathbf{z}_{t})
\end{align} 
as a consequence of optimization; it may have additional factorization within $q(\boldsymbol{\Omega})$ and $q(\mathbf{y}_{t})$ as well. 
Note that in our setting we are not performing inference over any global latent variables (shared across episodes) but rather latent variables at two levels of a hierarchical Bayesian model \citep{Gelman2004} whose latent variables are inferred separately for each episode $\mathbf{X}$. The underlying generative model's parameters are shared across episodes. 

\subsection{Objective}

As in AEVB, the objective is to maximize the evidence lower bound \citep{KingmaWelling2013}. For the models considered here, the ELBO has the form 
\begin{align}
\mathbb{E}_{q(\mathbf{Z})q(\mathbf{Y})q(\boldsymbol{\Omega})} \ln{p(\mathbf{X},\mathbf{Z},\mathbf{Y},\boldsymbol{\Omega})} - \ln{q(\mathbf{Z})q(\mathbf{Y})q(\boldsymbol{\Omega})}. 
\end{align}
During the forward pass, we perform the hybrid inference algorithm described in the previous section. Subsequent to running the inference algorithm during the forward pass, the inference model reduces to $q(\mathbf{Z}, \mathbf{Y}, \boldsymbol{\Omega}) = q(\boldsymbol{\Omega})\prod_{t=1}^{T}q(\mathbf{y}_{t})q(\mathbf{z}_{t})$ at which point the evidence lower bound can be expressed as 

\begin{equation}
\begin{split}
\mathcal{L}_{\theta,\phi}(X) = \, &\mathbb{E}_{q(\boldsymbol{\Omega})} \sum_{t=1}^{T} \mathbb{E}_{q(\mathbf{y}_{t})q(\mathbf{z}_{t})} \big{[} \ln{p(\mathbf{x}_{t}|\mathbf{z}_{t})} - D_{\text{KL}}(q(\mathbf{z}_{t})||p(\mathbf{z}_{t}|\mathbf{y}_{t},\boldsymbol{\Omega})) - D_{\text{KL}}(q(\mathbf{y}_{t})||p(\mathbf{y}_{t}|\boldsymbol{\Omega})) \big{]} \\ 
&- D_{\text{KL}}(q(\boldsymbol{\Omega})||p(\boldsymbol{\Omega})) 
\end{split}
\end{equation}
which is more computationally expedient to work with. A detailed derivation is available in Appendix A. 
Note that for the models we study in this paper, all the KL divergences (and expected KL divergences) appearing in the ELBO have closed-form expressions. In particular, if the prior and posterior for memory are matrix-variate Gaussian distributions, their KL divergence has a closed form expression (Appendix D.1). 

For conditional KL divergence terms, it is also possible in many cases to apply variants of the reparametrization trick to the continuous conditioning variables \citep{KingmaWelling2013, Rezende2014, Figurnov2018}. In particular, there is a reparametrization trick for multivariate Gaussians with full covariance \citep{Rezende2014}. This can also be applied to latent variables with matrix-variate Gaussian distributions, since these random variables can be identified with their vectorized counterparts (Appendix D.2). 

We use a weighted sum to compute the expectation of any conditional KL divergence terms with respect to any discrete latent variables in $\mathbf{y}_{t}$, such as mixture assignments. Note that, thanks to the mean-field factorization of our inference model, $q(\mathbf{z}_{t}, \mathbf{y}_{t}) = q(\mathbf{z}_{t})q(\mathbf{y}_{t})$, we do not have to decode a separate code $\mathbf{z}_{t}$ for each possible setting of these discrete variables. 

\subsection{Training Algorithm}

We train our generative model and recognition model using the AEVB algorithm \citep{KingmaWelling2013}. 
In AEVB, training on a batch of data is a two-step process, slightly reminiscent of variational EM \citep{Jordan1999}. For our model, training proceeds as follows:
\begin{enumerate}
\item In the E-like step, data is processed via a forward pass through the computation graph. During the forward pass, the recognition model sets the variational parameters of each $q(\mathbf{z}_{t})$, and these are used to iteratively tune the variational parameters of all $q(\mathbf{y}_{t})$ and $q(\boldsymbol{\Omega})$ using mean-field variational Bayesian methods (described in Appendix B). 
\item In the M-like step, we backpropagate a gradient estimate of the evidence lower bound through the computation graph. 
Using the reparameterization trick \citep{KingmaWelling2013}, we obtain unbiased gradient estimates of the evidence lower bound with respect to the generative model parameters $\theta$ and the recognition model parameters $\phi$. These gradient estimates are used to train the generative model and the recognition model. 
\end{enumerate}

Note that the variational parameters of $q(\mathbf{Z})$, $q(\mathbf{Y})$, $q(\boldsymbol{\Omega})$ are all computationally dependent on the recognition model's outputs. We wish to train $\phi$ so as to minimize the KL divergence between the inference distribution and the true posterior: 
\begin{align}
D_{\text{KL}}(q_{\phi}(\mathbf{Z})q_{\phi}(\mathbf{Y})q_{\phi}(\boldsymbol{\Omega})||p_{\theta}(\mathbf{Z},\mathbf{Y},\boldsymbol{\Omega}|\mathbf{X}))
\end{align}
While the KL divergence itself is intractable, we can obtain an unbiased estimate of its negative gradient by using the reparametrization trick on the evidence lower bound, and backpropagating through the memory writing algorithm.
\section{Discussion}

\subsection{Theoretical Discussion}

In principle, many variations on the above idea are possible. 

Measuring generalization in the outer loop objective of meta-learning algorithms is a common training strategy \citep{Finn2017, Garnelo2018}, and can be applied here as well. For example, we could consider adapting the episode-level latent variables using only a strict subset of the observations in the episode $\mathbf{X}^{\prime} \subset \mathbf{X}$. We could then compute the evidence lower bound for the entire episode $\mathbf{X}$ by applying the recognition model to select variational parameters for $q(\mathbf{z}_{t})$ and applying mean-field variational Bayes to the distribution $q(\mathbf{y}_{t})$. This would measure a certain type of generalization of the episode-level latent variables' variational distribution $q(\boldsymbol{\Omega})$, since these would have only been tuned on a strict subset of the episode. 

It may also be possible to interpolate between training curricula measuring generalization and those that do not. For example, we could use a training curriculum based on online VB. This would have the  effect of making some observations' contribution to the episode-level latent variables slightly more `stale' than others, which could be used to encourage generalization among the encodings of each observation. A similar approach is used by \citet{Wu2018b}, though the authors use a heuristic inference algorithm, not mean-field variational Bayes. The limitations of their algorithm are demonstrated in our experiments in the meta-learned episodic setting, and we also observe an online inference algorithm to underperform an offline one, which suggests that an offline inference algorithm should be used when possible. An alternative would be to dispense with the variables $\mathbf{y}_{t}$ from the model, so that better theoretical guarantees could be obtained in the online setting (Cf. \citet{NealHinton1998}). 

Finally, it is possible in some cases to `collapse out' some of the latent variables \citep{Teh2007}. 
Unfortunately, training on the collapsed evidence lower bound can have practical problems. For instance, if we can collapse out the episode-level latent variables, an inadvertent consequence is that the inference model is never trained to produce codes which alias well in memory. Put simply, some type of inference distribution for $\boldsymbol{\Omega}$ will still be needed at test time, and the codes sampled from the generative model $\mathbf{z} \sim p(\mathbf{z}|\mathbf{y},\boldsymbol{\Omega})$ may not resemble those from the episode. 

\subsection{Practical Discussion} 

In practice, we found several techniques were necessary to achieve top results. We detail these here. 

\subsubsection{Dynamic Memory Initialization}

In our fast-adapting mixture model setting, the variational mean of each cluster location is initialized using the k-means++ initialization \citep{Arthur2007}, which is a randomized algorithm typically used to seed Lloyd's algorithm. Our memory writing algorithm then runs as normal, and obtains significantly better results. We found this initialization was beneficial to use during training, not just at test time. This initialization is not backpropagated through. 

\subsubsection{Stable Representations}

To obtain stable representations for our dynamic memory initialization, it is sensible to avoid the use of batch normalization. To speed up training in our mixture-based memory models, we instead use group normalization \citep{YxWuHe2018} in the encoder and decoder. Further, we used the Swish-1 nonlinearity \citep{Ramachandran2017}, and reduced the Adam optimizer hyperparameter $\beta_{1}$ from $0.9$ to $0.0$. 

\subsubsection{Novel Stochastic Regularizer}

We found the sample quality of the observations generated from memory could be improved by incorporating a novel stochastic regularizer. 
Our regularizer modifies the variational parameters of perceptual codes. 
The regularizer is only applied after memory writing, so that inference is unaffected by these modifications. 

Our regularizer applies to models whose conditional priors $p(\mathbf{z}_{t}|\mathbf{y}_{t},\boldsymbol{\Omega})$ for perceptual codes are spherical Gaussians with a constant variance, and whose recognition model parametrizes a fixed-form diagonal Gaussian distribution $q(\mathbf{z}_{t})$ over perceptual codes. Our regularizer replaces the mean of the variational distribution $q(\mathbf{z}_{t})$ with a random convex combination $t \boldsymbol{\mu}_{q,z_{t}} + (1- t) \boldsymbol{\mu}_{p,z_{t}}$ of the original mean $\boldsymbol{\mu}_{q,z_{t}}$ of $q(\mathbf{z}_{t})$ and the expectation $\boldsymbol{\mu}_{p,z_{t}} := \mathbf{E}_{q(\boldsymbol{\Omega})q(\mathbf{y}_{t})}[\boldsymbol{\mu}_{p,z_{t}}(\boldsymbol{\Omega},\mathbf{y}_{t})]$ of the mean parameter for the distribution $p(\mathbf{z}_{t}|\mathbf{y}_{t},\boldsymbol{\Omega})$. We sampled $t := \gamma - \epsilon + \delta * s$ for $s \sim \text{Beta}(\alpha, \beta)$ where $\gamma, \epsilon, \delta, \alpha, \beta$ are hyperparameters.

\section{Experiments}

The code for our experiments is available at \texttt{https://github.com/lucaslingle/metavb}

\subsection{Evaluation Settings and Datasets}


To evaluate the performance of fast-adapting latent variable models, we believe the proper benchmark is the test-set evidence lower bound. It is also possible to use conditional objectives, but these do not correspond to a lower bound on the log-likelihood of the model, and do not adequately reflect the nature of the free energy surface formed by the evidence lower bound. Thus, we use the evidence lower bound for the episode. 


To facilitate comparison with non-episodic models like VAEs, we will compare models' ELBO on episodes rather than individual frames. To simplify comparison for future works by allowing comparison between different episode lengths, in practice we report the ELBO divided by the episode length (`ELBO per frame'). 

For evaluation, we extensively use three types of data: synthetic data, the CIFAR-10 dataset \citep{Krizhevsky2009}, and the CelebFace Attributes dataset \citep{Liu2015}, and further information is given in Appendix E. 

\subsection{Benchmarking the Inference Algorithms}

First we investigate inferential performance with application to a type of linear Gaussian model with a matrix-variate Gaussian prior on the observation matrix. Similarly structured models can serve as useful submodules of higher-capacity models (e.g., our models in Appendix C.4-C.5), and so the comparative merit of various inference algorithms in this setting has broader implications. 

For these experiments, we generate synthetic episodes, and then we use this data to benchmark inferential performance using two algorithms. The first algorithm will use our coordinate-ascent variational inference approach, based on mean-field variational Bayes. The second algorithm will be the DKM algorithm given by \citet{Wu2018b}. 

Our coordinate-ascent variational inference approach, based on mean-field variational Bayes, consistently outperformed the DKM algorithm (Fig. 2). 

\begin{figure}[h]
\includegraphics[height=5.0cm]{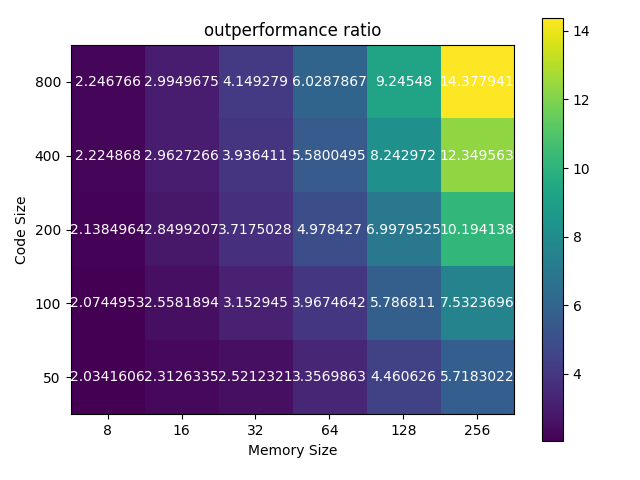} 
\includegraphics[height=5.0cm]{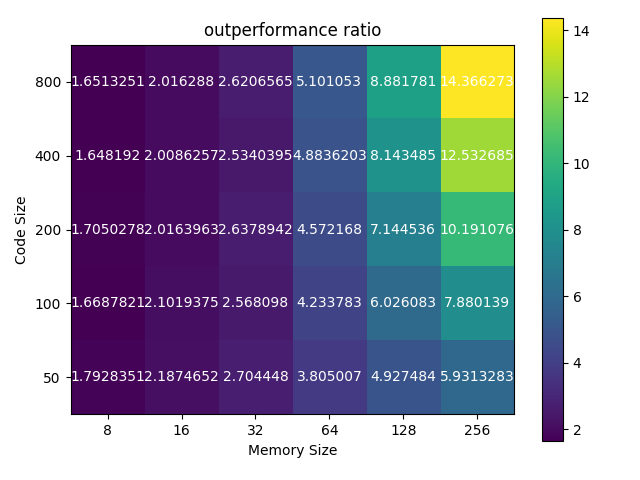} 
\caption{
Left and Right: Heatmaps showing the outperformance ratio by our algorithm on the evidence lower bound vs (a) DKM online non-iterative,  (b) DKM batched iterative. Results of each square within each heatmap used a single distinct episode of synthetic data, of length $T = 32$.} 
\end{figure}

In Figure 2a, it can be seen that our standard algorithm obtained an ELBO between 2x to 14x better than the DKM algorithm on the same data. In Figure 2b, we show an iterative offline algorithm, similar in spirit to the DKM algorithm. This algorithm did better, but ours still outperformed it by a minimum of $64\%$. 

These results suggest that our proposed approach may be a more reliable choice if performance on the evidence lower bound is desired. This distinction may be particularly meaningful at test time, since it may also indicate a misallocation of probability density by the underperforming algorithm. In deep generative modeling settings, we conjecture that this could lead to poorer quality observations being generated or retrieved, as well as decreased coverage of the observations in the episode. 

\subsection{Benchmarking on Standard Objectives}

We now benchmark our proposed approach in the context of deep generative models. Our goal in this section is to test if our inferential approach improves the sample efficiency and training stability of deep generative models trained on the standard evidence lower bound objective. 

For this experiment we apply our proposed approach to two simple generative models. We use simple models in this section to control for the complexity of the generative model; our baselines will be similar in complexity. 

These models use one top-level latent variable, $\mathbf{\Omega} = \{ \mathbf{M} \}$, with a matrix-variate Gaussian prior $ \mathcal{MN}_{K \times C}(\mathbf{M}|\mathbf{0}_{K \times C}, \mathbf{I}_{K}, \mathbf{I}_{C})$; they use local latent variables $\mathbf{y}_{t} = \{ \mathbf{w}_{t} \}$, which have either a Gaussian distribution or a one-hot categorical distribution; and they use perceptual codes defined by $p(\mathbf{z}_{t}|\mathbf{w}_{t},\mathbf{M}) = \mathcal{N}_{C}(\mathbf{z}_{t}|\mathbf{M}^{\top}\mathbf{w}_{t}, \sigma_{z}^{2}\mathbf{I}_{C})$, and use a decoder neural network as usual. 

\subsubsection{Baselines}

We implemented and benchmarked against: (1) the variational autoencoder (VAE) \citep{KingmaWelling2013}; (2) the Neural Statistician (NS) \citep{Edwards2017}; and (3) the Dynamic Kanerva Machine (DKM) \citep{Wu2018b}. 

Our implementations use the same architectural specification for the encoder and decoder of each model. Where required, additional transformations of the encoder output are made, in order to compute variational parameters. Architecture details and hyperparameters are given in App. F. 

\subsubsection{Experimental Results}

In this section, we report the quantitative results for our models and those we have benchmarked against. 

\begin{table}[h!]
\centering
\scalebox{0.8}{
\begin{tabular}{ c | c c } 
 \hline \\ [-1.5ex]
 Model & CIFAR-10; Test Set ELBO per frame \\
 \hline \\ [-1.5ex]
  VAE \citep{KingmaWelling2013} & \bf{3508.5} \\
  NS \citep{Edwards2017} & 2178.5 \\
  DKM \citep{Wu2018b} & 2198.3 \\
 \bf{VBM, Basic, Gaussian (ours)} & \bf{3451.7} \\
 \bf{VBM, Basic, Categorical (ours)} & \bf{3593.7} \\
 \hline
 \hline \\ [-1.5ex]
  Model & CelebA; Test Set ELBO per frame \\
  \hline \\ [-1.5ex]
  VAE \citep{KingmaWelling2013} & \bf{4174.8} \\
  NS \citep{Edwards2017} & 3242.3 \\
  DKM \citep{Wu2018b} & 2868.0 \\
  \bf{VBM, Basic, Gaussian (ours)} & \bf{4155.3} \\
  \bf{VBM, Basic, Categorical (ours)} & \bf{4422.5} \\
 \hline
\end{tabular}}
\caption{ELBO per frame on CIFAR-10 and CelebA. Each model was randomly initialized and trained independently.} 
\end{table}

\emph{VAE.} The training process for the VAE was stable and sample efficient. The VAE model is a non-episodic model and can be trained with a fairly small batch size, but for fair comparison with the other models, we train it using a batch of episodes in similar fashion to the episodic models. 

\emph{Neural Statistician.} The Neural Statistician's performance was very stable, but asymptoted on both datasets as training progressed. In later epochs, the training set performance improved very slowly, but a slight continued improvement on the validation set allowed it to avoid early stopping for the entire duration. 
We note that compared to the model used by \citet{Edwards2017}, we added an additional nonlinearity in the pooling layer, before computing the approximate posterior over the context variable's variational parameters with a linear layer. We found our Neural Statistician was otherwise unstable in our unstructured episodic setting. 

\emph{Dynamic Kanerva Machine.} The DKM became intermittently unstable as training progressed. This occurred on both datasets. The DKM was the only model to stop early. It stopped early on both datasets. We note that \citet{Wu2018b} used a biased estimator to train their models, approximating samples of the variational posterior for memory using its mean, see Sec 3 of their paper. For fair comparison with the other models, we used an unbiased estimator for training and benchmarking of all models, including the DKM. 

\emph{Our Models.} 
Interestingly, our models trained very stably and obtained final performance comparable to the VAE, while outperforming the other models by a large margin. The \emph{basic Gaussian} model is quite similar to the DKM in its graphical model specification, and yet our model obtains significantly better results. We attribute this result to the noisy input to the decoder in the DKM model, which arises from our use of an unbiased estimator. The \emph{basic Categorical} model was a bit slower to train with the basic hyperparameter setup than the VAE or Gaussian model, but we found that the per-epoch training progress could be sped up by halving the batch size and duplicating the episodes processed within each batch; with this trick, the basic Categorical model obtains competitive results with the VAE and our basic Gaussian model in the same number of maximum training epochs.\footnote{Since our goal is to measure sample efficiency rather than the efficiency of each gradient step, we consider this modification to be a reasonable one.} Surprisingly, while the Neural Statistician also trained slowly, we found this modified training curriculum did not improve its performance. 

In summary, our models obtain good quantitative performance and train steadily like the VAE, and much more efficiently than the other baseline episodic generative models. Since these episodic models were similar in terms of both encoder-decoder architecture and graphical model structure--but were based on other inference algorithms--we believe these results demonstrate a remarkable reliability by mean-field variational Bayes from the standpoints of quantitative performance, training stability, and sample efficiency. 

\subsection{Qualitative Evaluation}

In this section, we present qualitative results for our more sophisticated models: a \emph{scalable mixture-based memory model} (App. C.4) 
and a \emph{tree-structured memory model} (App. C.5). 

The models are both trained on the CIFAR-10 training dataset, using the improved encoder-decoder architecture, the stochastic regularizer, and the dynamic memory initialization. Training hyperparams and further details are given in App. F.2. 

\subsubsection{Generating from Memory}

Let's investigate the samples generated from a memory state $q(\boldsymbol{\Omega})$, when the memory is written to using test set data. For this experiment, we generate from memory using the algorithm described in App. G. 
Our results are shown below (Fig. 3). 

\begin{figure}[h!]
\includegraphics[height=6.00cm]{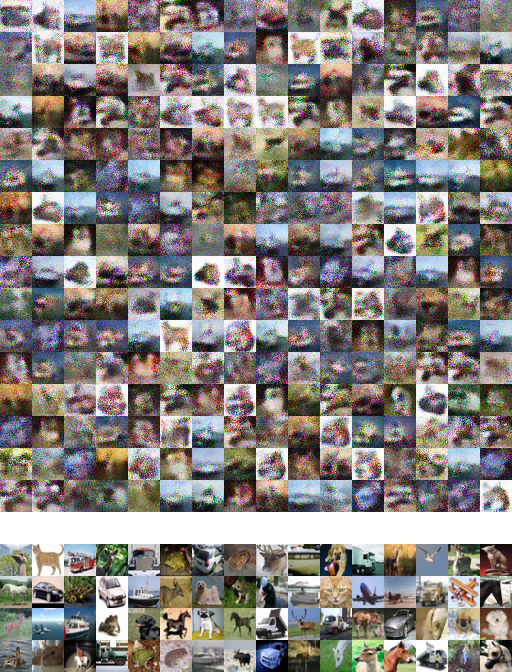} 
\hspace{0.5cm}
\includegraphics[height=6.00cm]{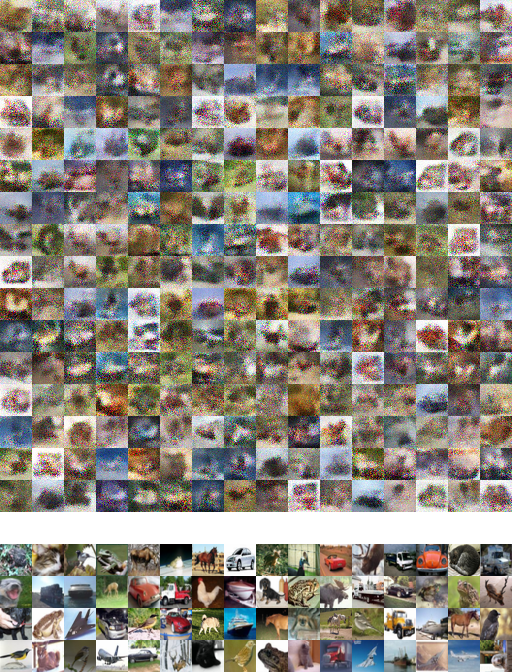} 
\caption{
Observations generated from memory, based on a reference episode written into memory. Reference episodes shown below each grid of samples. Left: Scalable Mixture-Based Memory Model. Right: Tree-Structured Memory Model.}
\end{figure}

For the scalable mixture-based memory model, the observations generated from memory are crisp, diverse, coherent, and largely depict the same entities, scenes, and patterns as the observations in each episode. Samples for additional episodes are shown in App. H, 
and illustrate the same high degree of \emph{sample fidelity}. Our results improve on those of prior works (see App. H), 
generally allowing greater recall of diverse episodes, and generating more coherent observations in more cases. 

By inspection, the tree-structured memory model generates samples that tend to be more blurry by default. Interestingly, despite the fact that the model can in principle assign the partitions of a code into different clustering patterns--a degree of freedom absent from the scalable mixture-based memory model--the result when naively generating from memory is worse, in this seemingly more powerful model. This can be understood as a consequence of the fact that we are no longer modeling the covariance between all code components, but only those within each code partition. (This generative model essentially specifies a block diagonal covariance matrix for codes, but the blocks can be swapped out.) Since the variability among different partitions of a code is no longer directly coordinated, blurry images may result. It may be possible to resolve this by conditioning all the addressing weight priors for a given timestep (App. C.5, 
Eq. 227) on yet another latent variable, which could represent their mean; we leave this as a possible direction for future work. 

Thus, the scalable mixture-based memory model can generate some sharp images directly, but the tree-structured memory model will require additional techniques in order to do so. In the next section, we discuss a simple technique, applicable to both models, that allows the proposed tree-structured memory model to generate more compelling samples, and also improves the results of the scalable mixture model. 

\subsubsection{Iterative Reading}

In several prior works \citep{Hinton2006, Wu2018a}, a variety of iterative sampling algorithms have been proposed and empirically observed to improve sample quality of fast-adapting generative models. In this section, we show samples generated by an iterative sampling algorithm of our own design, and demonstrate that the samples generated directly from memory in our models are similarly amenable to improvement via the use of our proposed algorithm. 

For this experiment, we iteratively sample from memory using the algorithm described in App. G. 
Our results are shown in Fig. 4. 

\begin{figure}[h]
\includegraphics[height=5.0cm, width=5.0cm]{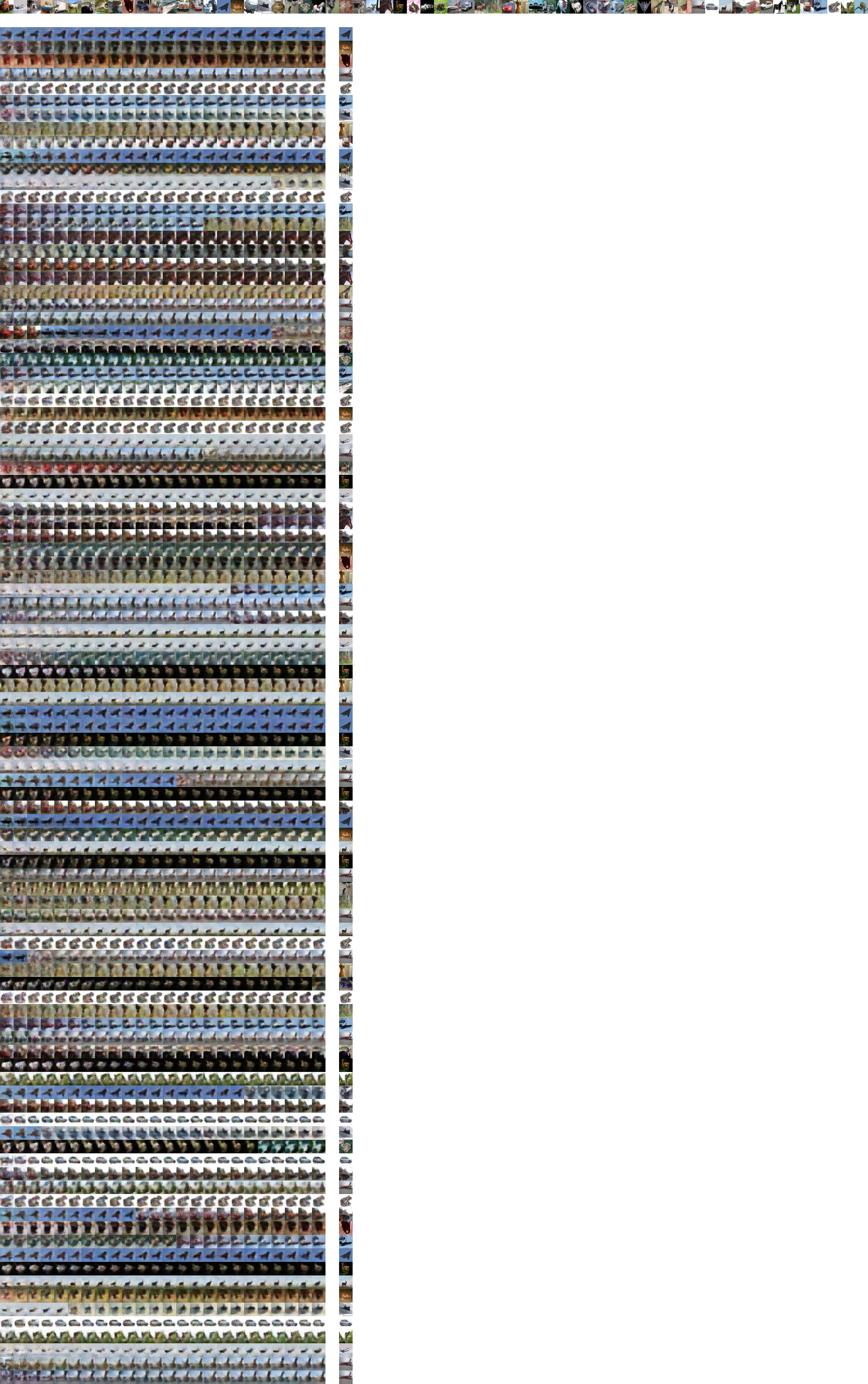} 
\includegraphics[height=5.0cm, width=5.0cm]{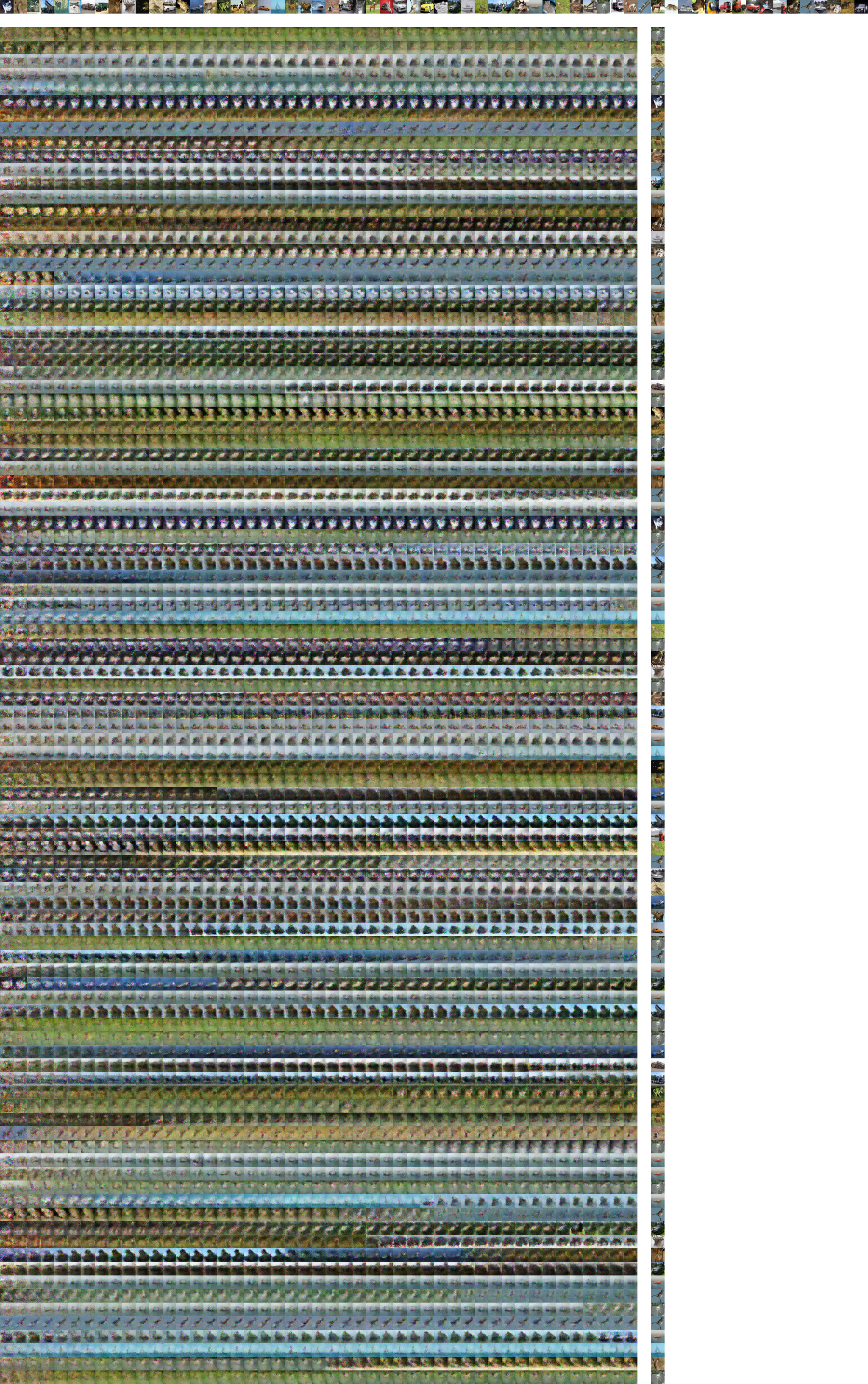} 
\caption{
Observations iteratively generated from memory, based on a reference episode written into memory. Reference episode is shown above each grid of iterative samples, and nearest neighbor of the last sample in each sequence is shown in column to the right. Left: Scalable Mixture-Based Memory Model. Right: Tree-Structured Memory Model. (Best seen electronically.)}
\end{figure}

\subsubsection{Resizing Memory}

In our models, the latent space of perceptual codes is organized, but the explicit memory layout is determined on-the-fly by the dynamic memory initialization and the subsequent optimization via mean-field variational Bayes. This makes resizability of our scalable mixture-based memory model simple. For instance, in TensorFlow \citep{Abadi2016} we can resize our trained memory models simply by instantiating a model from the same class and with the same variable scope, and restoring the trained variables from a checkpoint. By changing the field for the model size in the class, a larger memory model can be obtained. The only checkpointed variables are the neural network parameters, and in contrast to \citet{Rao2019, Marblestone2020}, these networks have no parametric dependency on the memory size or number of clusters. Since the VB updates depend only on perceptual codes themselves and not directly any neural network parameters, they are implemented to adapt automatically. This facilitates increased capacity and fast writes without any gradient-based training of a large memory model. 

In this section, we investigate the sample quality of the observations generated from memory, when both the episode length and the memory size are grown proportionally. Using the same model from the previous sections, we consider an episode length $T = 640$ and a number of clusters $H = 100$. 

Below, we display a visualization of the memory state, samples generated from the model, and iterative reads. 

\begin{figure}[h]
\includegraphics[height=10cm]{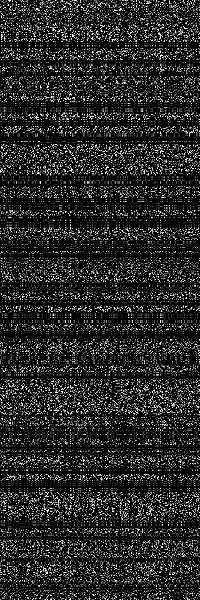} 
\includegraphics[height=10cm]{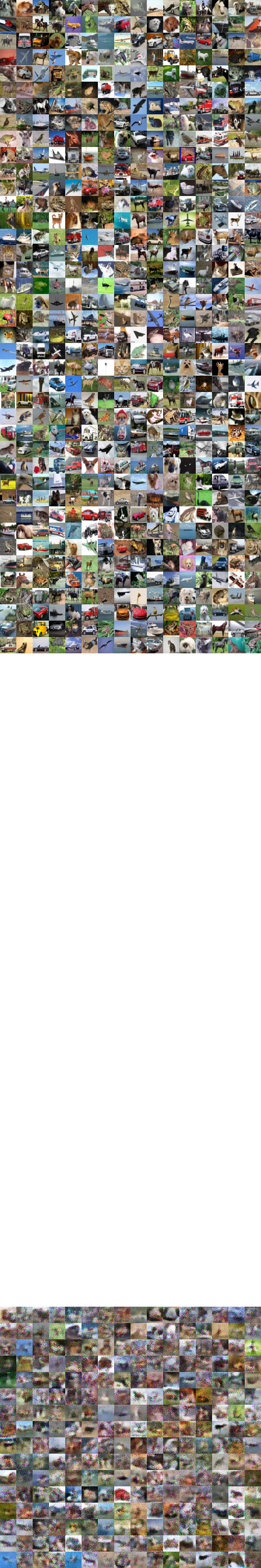} 
\includegraphics[height=10cm]{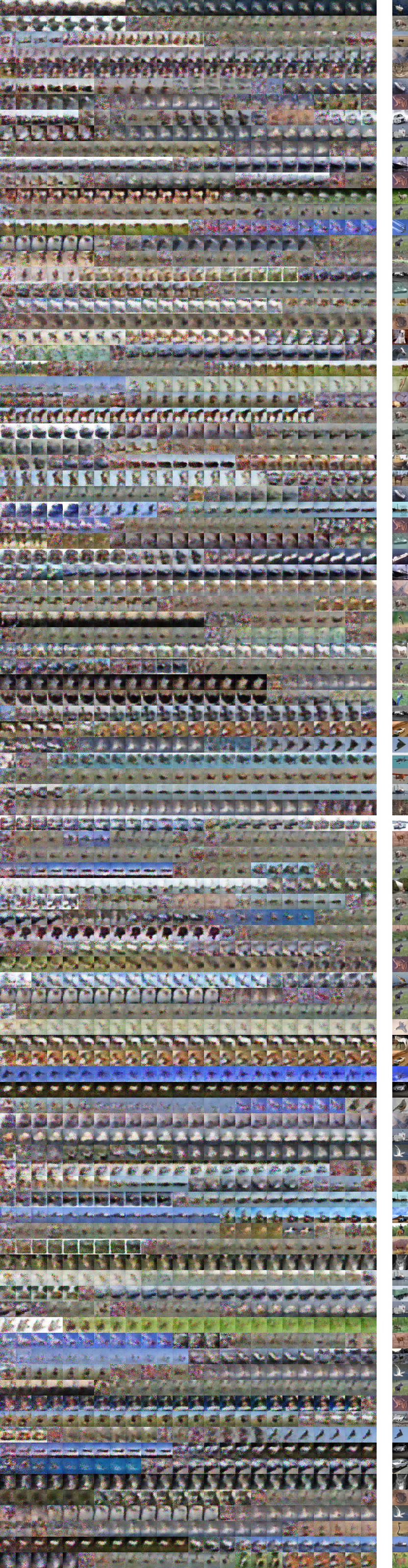} 
\caption{Left: Visualization of the memory state for the scalable mixture-based memory model with $H = 100$ clusters, created by running mean-field variational Bayes on a test-set episode of length $T = 640$, and concatenating the mean $\mathbf{R}_{h}$ of each matrix-variate Gaussian memory distribution $q(\mathbf{M}_{h})$ together on the vertical axis, for $h = 1, \ldots H$; each matrix had $K = 6$ rows. Center: Samples generated directly from memory. Right: Iteratively generated samples.}
\end{figure}

In App. H, 
we also show samples for an even larger episode, of length $T = 1280$, with memory scaled sublinearly to $H = 150$ clusters, obtaining comparable or even better results to the ones shown above. 

\section{Conclusion}

In this work we have introduced a scalable method for designing fast-adapting generative models. For the results presented in this paper, we have used nearly i.i.d. episodic data, and yet our models train stably and maintain an expressive distribution over latent space. The qualitative results suggest that our models are capable of generating reasonable quality observations resembling test-time data. Moreover, the models presented in this paper have no parametric dependency on the memory size, and thus it is easy to resize our models at test time. 

Our contribution leverages the AEVB framework and mean-field variational Bayes, and creates fast-adapting latent-space generative models. It is made possible by a new result, showing that the relevant VB updates, for deep conditionally independent hierarchical models with neural network decoders, do not depend on this neural network decoder. In general, the VB updates for any given latent variable \emph{only} depend on the generative model through the expectation of the natural parameters for that latent variable, conditioned on the other latent variables in its Markov boundary. Consequently, the perceptual codes serve as a `buffer' between the decoder and the other latent variables, facilitating tractable VB updates. 

Some possible directions for future work would be as follows. (1) Improving the encoder-decoder architecture. (2) Improving the ability of the memory module to store perspective-invariant representations rather than fully `perceptual' ones. (3) Improving the ability of the memory module to generalize. (4) Improving the ability of the memory module to incorporate temporal conditioning information. (5) Investigating the viability of the approach in other domains. (6) Investigating the viability of performing Bayesian model selection in latent space. 

\clearpage 

\bibliographystyle{iclr2020_conference} 
\bibliography{ms} 

\appendix

\section{Evidence Lower Bound}

In this section, we derive the evidence lower bound for our model. 

\begin{align}
\ln p(\mathbf{X}) &= \ln \int p(\mathbf{X},\mathbf{Z},\mathbf{Y},\boldsymbol{\Omega}) d\mathbf{Z}d\mathbf{Y}d\boldsymbol{\Omega} \\
&= \ln \int q(\mathbf{Z})q(\mathbf{Y})q(\boldsymbol{\Omega}) \frac{p(\mathbf{X},\mathbf{Z},\mathbf{Y},\boldsymbol{\Omega})}{q(\mathbf{Z})q(\mathbf{Y})q(\boldsymbol{\Omega})} d\mathbf{Z}d\mathbf{Y}d\boldsymbol{\Omega} \\
&\geq \int q(\mathbf{Z})q(\mathbf{Y})q(\boldsymbol{\Omega}) \ln \frac{p(\mathbf{X},\mathbf{Z},\mathbf{Y},\boldsymbol{\Omega})}{q(\mathbf{Z})q(\mathbf{Y})q(\boldsymbol{\Omega})} d\mathbf{Z}d\mathbf{Y}d\boldsymbol{\Omega} \\
&= \mathbb{E}_{q(\mathbf{Z})q(\mathbf{Y})q(\boldsymbol{\Omega})} \ln{p(\mathbf{X},\mathbf{Z},\mathbf{Y},\boldsymbol{\Omega})} - \ln{q(\mathbf{Z})q(\mathbf{Y})q(\boldsymbol{\Omega})} 
\end{align}

For purposes of training via stochastic gradient ascent, it will be useful to simplify the evidence lower bound further. 

Three remarks are in order:
\begin{itemize}
\item In our generative model, the local latent variables at each timestep are conditionally independent from those at the other timesteps, given memory $\boldsymbol{\Omega}$. 
\item As shown in Appendix B, the variational distribution $q(\mathbf{Y})$ factorizes as $q(\mathbf{Y}) = \prod_{t=1}^{T} q(\mathbf{y}_{t})$ as a consequence of optimization. 
\item A similar derivation holds for $q(\mathbf{Z})$, except that each $q(\mathbf{z}_{t})$ given by variational Bayes does not have an analytically tractable density. We therefore use a recognition model, and design it so as to compute $q(\mathbf{Z})$ without any stochastic dependencies between the codes at different timesteps. Consequently, we have $q(\mathbf{Z}) = \prod_{t=1}^{T} q(\mathbf{z}_{t})$. 
\end{itemize}

Thus, after applying the recognition model and running the memory writing algorithm, the evidence lower bound simplifies to 

\begin{align}
\ln p(\mathbf{X}) &\geq \mathbb{E}_{q(\mathbf{Z})q(\mathbf{Y})q(\boldsymbol{\Omega})}{\bigg{[}\ln{p(\mathbf{X},\mathbf{Z},\mathbf{Y},\boldsymbol{\Omega})} - \ln{q(\mathbf{Z})q(\mathbf{Y})q(\boldsymbol{\Omega})}\bigg{]}} \\
&= \mathbb{E}_{q(\mathbf{Z})q(\mathbf{Y})q(\boldsymbol{\Omega})}{\bigg{[}\ln{p(\mathbf{X},\mathbf{Z},\mathbf{Y}|\boldsymbol{\Omega})} + \ln{p(\boldsymbol{\Omega})} - \ln{q(\boldsymbol{\Omega})} - \ln{q(\mathbf{Z})q(\mathbf{Y})}\bigg{]}} \\
&= \mathbb{E}_{q(\mathbf{Z})q(\mathbf{Y})q(\boldsymbol{\Omega})}{\bigg{[}\ln{p(\mathbf{X},\mathbf{Z},\mathbf{Y}|\boldsymbol{\Omega})} - \ln{q(\mathbf{Z})q(\mathbf{Y})}\bigg{]}} - D_{\text{KL}}(q(\boldsymbol{\Omega})||p(\boldsymbol{\Omega})) \\
&= \mathbb{E}_{q(\mathbf{Z})q(\mathbf{Y})q(\boldsymbol{\Omega})}{\bigg{[}\sum_{t=1}^{T}\ln{p(\mathbf{x}_{t},\mathbf{z}_{t},\mathbf{y}_{t}|\boldsymbol{\Omega})} - \ln{q(\mathbf{z}_{t})q(\mathbf{y}_{t})} \bigg{]}} - D_{\text{KL}}(q(\boldsymbol{\Omega})||p(\boldsymbol{\Omega})) \\
&= \mathbb{E}_{q(\mathbf{Z})q(\mathbf{Y})q(\boldsymbol{\Omega})}{\bigg{[}\sum_{t=1}^{T}\ln{p(\mathbf{x}_{t}|\mathbf{z}_{t})p(\mathbf{z}_{t}|\mathbf{y}_{t},\boldsymbol{\Omega})p(\mathbf{y}_{t}|\boldsymbol{\Omega})} - \ln{q(\mathbf{z}_{t})q(\mathbf{y}_{t})} \bigg{]}} - D_{\text{KL}}(q(\boldsymbol{\Omega})||p(\boldsymbol{\Omega})) 
\end{align}

Thus, 
\begin{equation}
\begin{split}
\ln{p(\mathbf{X})} \geq \, &\mathbb{E}_{q(\boldsymbol{\Omega})} \sum_{t=1}^{T} \mathbb{E}_{q(\mathbf{y}_{t})q(\mathbf{z}_{t})} \big{[} \ln{p(\mathbf{x}_{t}|\mathbf{z}_{t})} - D_{\text{KL}}(q(\mathbf{z}_{t})||p(\mathbf{z}_{t}|\mathbf{y}_{t},\boldsymbol{\Omega})) - D_{\text{KL}}(q(\mathbf{y}_{t})||p(\mathbf{y}_{t}|\boldsymbol{\Omega})) \big{]} \\ 
&- D_{\text{KL}}(q(\boldsymbol{\Omega})||p(\boldsymbol{\Omega})) 
\end{split}
\end{equation}

This concludes the derivation. 
\section{Variational Bayes}
\label{appdx:B}

In this section we use variational calculus to derive formulae for updating the variational distributions for addresses and memory. 

Following \citet{Beal2003}, we cast inference as a constrained optimization problem in function space, apply the method of Lagrange multipliers, take functional derivatives w.r.t. each variational distribution separately, and equate these to zero. This allows us to obtain an update formula for each variational distribution that maximizes the evidence lower bound when the other variational distributions are held fixed. By iteratively applying these updates, we perform a form of coordinate ascent in function space, with each iteration yielding monotonic improvement in the evidence lower bound. 

Write 
\begin{align}
\mathcal{F}_{X}[q_{Z}, q_{Y}, q_{\Omega}] &:= \int q_{Z}(\mathbf{Z})q_{Y}(\mathbf{Y})q_{\Omega}(\boldsymbol{\Omega}) \ln \frac{p(\mathbf{X},\mathbf{Z},\mathbf{Y},\boldsymbol{\Omega})}{q_{Z}(\mathbf{Z})q_{Y}(\mathbf{Y})q_{\boldsymbol{\Omega}}(\boldsymbol{\Omega})} d\mathbf{Z}d\mathbf{Y}d\boldsymbol{\Omega} \\
\mathcal{G}_{Y}[q_{Y}] &:= \int q_{Y}(\mathbf{Y}) d\mathbf{Y} \\
\mathcal{G}_{\Omega}[q_{\Omega}] &:= \int q_{\Omega}(\boldsymbol{\Omega}) d\boldsymbol{\Omega} 
\,
\end{align}

It will be useful to rewrite $\mathcal{F}_{X}$. We have
\begin{equation}
\begin{split}
\mathcal{F}_{X}[q_{Z}, q_{Y}, q_{\Omega}] = &\int q_{Z}(\mathbf{Z})q_{Y}(\mathbf{Y})q_{\Omega}(\boldsymbol{\Omega}) \ln{p(\mathbf{X},\mathbf{Z},\mathbf{Y},\boldsymbol{\Omega})} d\mathbf{Z}d\mathbf{Y}d\boldsymbol{\Omega} \\
&- \int q_{Z}(\mathbf{Z})q_{Y}(\mathbf{Y})q_{\Omega}(\boldsymbol{\Omega}) \ln{q_{Z}(\mathbf{Z})q_{Y}(\mathbf{Y})q_{\Omega}(\boldsymbol{\Omega})} d\mathbf{Z}d\mathbf{Y}d\boldsymbol{\Omega} \\
\end{split}
\end{equation}
which reduces to
\begin{equation}
\begin{split}
\mathcal{F}_{X}[q_{Z}, q_{Y}, q_{\Omega}] = &\int q_{Z}(\mathbf{Z})q_{Y}(\mathbf{Y})q_{\Omega}(\boldsymbol{\Omega}) \ln{p(\mathbf{X},\mathbf{Z},\mathbf{Y},\boldsymbol{\Omega})} d\mathbf{Z}d\mathbf{Y}d\boldsymbol{\Omega} \\
&- \int q_{Z}(\mathbf{Z}) \ln{q_{Z}(\mathbf{Z})} d\mathbf{Z} \\
&- \int q_{Y}(\mathbf{Y}) \ln{q_{Y}(\mathbf{Y})} d\mathbf{Y} \\
&- \int q_{\Omega}(\boldsymbol{\Omega}) \ln{q_{\Omega}(\boldsymbol{\Omega})} d\boldsymbol{\Omega}
\end{split}
\end{equation}

\subsection{Variational Inference for Addresses}

We seek to optimize $\mathcal{F}_{X}[q_{Z}, q_{Y}, q_{\Omega}]$ with respect to the free distribution $q_{Y}$. This can be accomplished by the method of Lagrange multipliers. We form the Lagrangian expression
\begin{align}
\hat{\mathcal{F}}_{X}[q_{Z}, q_{Y}, q_{\Omega}] &:= \mathcal{F}_{X}[q_{Z}, q_{Y}, q_{\Omega}] + \lambda\bigg{(}\mathcal{G}_{Y}[q_{Y}] - 1\bigg{)} 
\end{align}
Taking the functional derivative of the Lagrangian expression with respect to $q_{Y}$, we have
\begin{equation}
\begin{split}
\frac{\delta \hat{\mathcal{F}}_{X}}{\delta q_{Y}} = &\int q_{Z}(\mathbf{Z})q_{\Omega}(\boldsymbol{\Omega}) \ln{p(\mathbf{X},\mathbf{Z},\mathbf{Y},\boldsymbol{\Omega})} d\mathbf{Z}d\boldsymbol{\Omega} \\
&- [\ln{q_{Y}(\mathbf{Y})} + 1] \\
&+ \lambda(1 - 0)
\end{split}
\end{equation}

Equating to zero and rearranging, we obtain
\begin{align}
\ln{q_{Y}(\mathbf{Y})} = \int q_{Z}(\mathbf{Z})q_{\Omega}(\boldsymbol{\Omega}) \ln{p(\mathbf{X},\mathbf{Z},\mathbf{Y},\boldsymbol{\Omega})} d\mathbf{Z}d\boldsymbol{\Omega} -1 + \lambda
\end{align}

Thus, 
\begin{align}
q_{Y}(\mathbf{Y}) &= \exp \bigg{\{} \int q_{Z}(\mathbf{Z})q_{\Omega}(\boldsymbol{\Omega}) \ln{p(\mathbf{X},\mathbf{Z},\mathbf{Y},\boldsymbol{\Omega})} d\mathbf{Z}d\boldsymbol{\Omega} -1 + \lambda \bigg{\}} \\
&\propto \exp \bigg{\{} \int q_{Z}(\mathbf{Z})q_{\Omega}(\boldsymbol{\Omega}) \ln{p(\mathbf{X},\mathbf{Z},\mathbf{Y},\boldsymbol{\Omega})} d\mathbf{Z}d\boldsymbol{\Omega}\bigg{\}} \\
&= \exp \bigg{\{} \int q_{Z}(\mathbf{Z})q_{\Omega}(\boldsymbol{\Omega}) [\ln{p(\boldsymbol{\Omega})} + \sum_{t=1}^{T} \ln{p(\mathbf{x}_{t},\mathbf{z}_{t},\mathbf{y}_{t}|\boldsymbol{\Omega})}] d\mathbf{Z}d\boldsymbol{\Omega}\bigg{\}} \\
&\propto \exp \bigg{\{} \int q_{Z}(\mathbf{Z})q_{\Omega}(\boldsymbol{\Omega}) [\sum_{t=1}^{T} \ln{p(\mathbf{x}_{t},\mathbf{z}_{t},\mathbf{y}_{t}|\boldsymbol{\Omega})}] d\mathbf{Z}d\boldsymbol{\Omega}\bigg{\}} \\
&= \exp \bigg{\{} \sum_{t=1}^{T} \int q_{Z}(\mathbf{z}_{t})q_{\Omega}(\boldsymbol{\Omega}) \ln{p(\mathbf{x}_{t},\mathbf{z}_{t},\mathbf{y}_{t}|\boldsymbol{\Omega})} d\mathbf{z}_{t}d\boldsymbol{\Omega}\bigg{\}} \\
&= \prod_{t=1}^{T} \exp \bigg{\{} \int q_{Z}(\mathbf{z}_{t})q_{\Omega}(\boldsymbol{\Omega}) \ln{p(\mathbf{x}_{t},\mathbf{z}_{t},\mathbf{y}_{t}|\boldsymbol{\Omega})} d\mathbf{z}_{t}d\boldsymbol{\Omega}\bigg{\}} \\
&= \prod_{t=1}^{T} \exp \bigg{\{} \int q_{Z}(\mathbf{z}_{t})q_{\Omega}(\boldsymbol{\Omega}) [\ln{p(\mathbf{x}_{t}|\mathbf{z}_{t})} + \ln{p(\mathbf{z}_{t},\mathbf{y}_{t}|\boldsymbol{\Omega})}] d\mathbf{z}_{t}d\boldsymbol{\Omega}\bigg{\}} \\
&\propto \prod_{t=1}^{T} \exp \bigg{\{} \int q_{Z}(\mathbf{z}_{t})q_{\Omega}(\boldsymbol{\Omega}) \ln{p(\mathbf{z}_{t},\mathbf{y}_{t}|\boldsymbol{\Omega})} d\mathbf{z}_{t}d\boldsymbol{\Omega}\bigg{\}} 
\end{align}

Where `$\propto$' denotes proportionality. Note that since $q_{Y}(\mathbf{Y})$ is constrained to integrate to 1, all terms from the exponent that do not vary with $\mathbf{Y}$ can be absorbed into the proportionality constant, without affecting the density. 

Thus we have shown that the generative neural network disappears from the expression for the variational addressing distribution. 
Before proceeding further, let us do the same for memory. 

\subsection{Variational Inference for Memory}

Now for $q_{\Omega}$. We seek to optimize $\mathcal{F}_{X}[q_{Z}, q_{Y}, q_{\Omega}]$ with respect to the free distribution $q_{\Omega}$. This can be accomplished by the method of Lagrange multipliers. We form the Lagrangian expression
\begin{align}
\hat{\mathcal{F}}_{X}[q_{Z}, q_{Y}, q_{\Omega}] &:= \mathcal{F}_{X}[q_{Z}, q_{Y}, q_{\Omega}] + \lambda\bigg{(}\mathcal{G}_{\Omega}[q_{\Omega}] - 1\bigg{)} 
\end{align}
Taking the functional derivative of the Lagrangian with respect to $q_{\Omega}$, we have
\begin{equation}
\begin{split}
\frac{\delta \hat{\mathcal{F}}_{X}}{\delta q_{\Omega}} = &\int q_{Z}(\mathbf{Z})q_{Y}(\mathbf{Y}) \ln{p(\mathbf{X},\mathbf{Z},\mathbf{Y},\boldsymbol{\Omega})} d\mathbf{Z}d\mathbf{Y} \\
&- [\ln{q_{\Omega}(\boldsymbol{\Omega})} + 1] \\
&+ \lambda(1 - 0)
\end{split}
\end{equation}

Equating to zero and rearranging, we have
\begin{align}
\ln{q_{\Omega}(\boldsymbol{\Omega})} = \int q_{Z}(\mathbf{Z})q_{Y}(\mathbf{Y}) \ln{p(\mathbf{X},\mathbf{Z},\mathbf{Y},\boldsymbol{\Omega})} d\mathbf{Z}d\mathbf{Y} -1 + \lambda
\end{align}

Thus, 
\begin{align}
q_{\Omega}(\boldsymbol{\Omega}) &= \exp \bigg{\{} \int q_{Z}(\mathbf{Z})q_{Y}(\mathbf{Y}) \ln{p(\mathbf{X},\mathbf{Z},\mathbf{Y},\boldsymbol{\Omega})} d\mathbf{Z}d\boldsymbol{\Omega} -1 + \lambda \bigg{\}} \\
&\propto \exp \bigg{\{} \int q_{Z}(\mathbf{Z})q_{Y}(\mathbf{Y}) \ln{p(\mathbf{X},\mathbf{Z},\mathbf{Y},\boldsymbol{\Omega})} d\mathbf{Z}d\mathbf{Y} \bigg{\}} \\
&= \exp \bigg{\{} \int q_{Z}(\mathbf{Z})q_{Y}(\mathbf{Y}) [\ln{p(\boldsymbol{\Omega})} + \sum_{t=1}^{T} \ln{p(\mathbf{x}_{t},\mathbf{z}_{t},\mathbf{y}_{t}|\boldsymbol{\Omega})}] d\mathbf{Z}d\mathbf{Y} \bigg{\}} \\
&= p(\boldsymbol{\Omega}) \exp \bigg{\{} \int q_{Z}(\mathbf{Z})q_{Y}(\mathbf{Y}) [\sum_{t=1}^{T} \ln{p(\mathbf{x}_{t},\mathbf{z}_{t},\mathbf{y}_{t}|\boldsymbol{\Omega})}] d\mathbf{Z}d\mathbf{Y} \bigg{\}} \\
&= p(\boldsymbol{\Omega}) \exp \bigg{\{} \sum_{t=1}^{T} \int q_{Z}(\mathbf{z}_{t})q_{Y}(\mathbf{y}_{t}) \ln{p(\mathbf{x}_{t},\mathbf{z}_{t},\mathbf{y}_{t}|\boldsymbol{\Omega})} d\mathbf{z}_{t}d\mathbf{y}_{t} \bigg{\}} \\
&= p(\boldsymbol{\Omega}) \prod_{t=1}^{T} \exp \bigg{\{} \int q_{Z}(\mathbf{z}_{t})q_{Y}(\mathbf{y}_{t}) \ln{p(\mathbf{x}_{t},\mathbf{z}_{t},\mathbf{y}_{t}|\boldsymbol{\Omega})} d\mathbf{z}_{t}d\mathbf{y}_{t} \bigg{\}} \\
&= p(\boldsymbol{\Omega}) \prod_{t=1}^{T} \exp \bigg{\{} \int q_{Z}(\mathbf{z}_{t})q_{Y}(\mathbf{y}_{t}) [\ln{p(\mathbf{x}_{t}|\mathbf{z}_{t})} + \ln{p(\mathbf{z}_{t},\mathbf{y}_{t}|\boldsymbol{\Omega})}] d\mathbf{z}_{t}d\mathbf{y}_{t} \bigg{\}} \\
&\propto p(\boldsymbol{\Omega}) \prod_{t=1}^{T} \exp \bigg{\{} \int q_{Z}(\mathbf{z}_{t})q_{Y}(\mathbf{y}_{t}) \ln{p(\mathbf{z}_{t},\mathbf{y}_{t}|\boldsymbol{\Omega})} d\mathbf{z}_{t}d\mathbf{y}_{t}\bigg{\}} 
\end{align}

Where `$\propto$' denotes proportionality. Note that since $q_{\Omega}(\boldsymbol{\Omega})$ is constrained to integrate to 1, all terms from the exponent that do not vary with $\boldsymbol{\Omega}$ can be absorbed into the proportionality constant, without affecting the density. 

Thus we have shown that the generative neural network disappears from the expression for the variational memory distribution. 

This derivation is generic, and shows that any conditionally independent hierarchical model of the form $p(\boldsymbol{\Omega})\prod_{t=1}^{T}p(\mathbf{x}_{t}|\mathbf{z}_{t})p(\mathbf{z}_{t}|\mathbf{y}_{t},\boldsymbol{\Omega})p(\mathbf{y}_{t}|\boldsymbol{\Omega})$ and any inference model of the form $q(\boldsymbol{\Omega})q(\mathbf{Y})q(\mathbf{Z})$, two properties hold: (1) the VB update for the joint distribution $q(\mathbf{Y})$ factors over timesteps, and (2) the VB updates for $q(\mathbf{y}_{t})$ and $q(\boldsymbol{\Omega})$ do not depend on $p(\mathbf{x}_{t}|\mathbf{z}_{t})$. 

A similar argument shows that the VB update for $q(\mathbf{Z})$ factors over timesteps, but unlike the other updates, this one would depend on $p(\mathbf{x}_{t}|\mathbf{z}_{t})$, so the VB updates for that distribution are intractable when $p(\mathbf{x}_{t}|\mathbf{z}_{t})$ is a neural network; thus we use a recognition model to set the variational parameters of $q(\mathbf{z}_{t})$ directly. 
\section{Variational Bayesian Update Rules}

\subsection{Variational Bayesian Update Rules: Gaussian Addresses}

\begin{theorem}
Consider a generative model of the form 
\begin{align*}
p(\mathbf{X},\mathbf{Z},\mathbf{W},\mathbf{M}) &= p(\mathbf{M})\prod_{t=1}^{T}p(\mathbf{w}_{t})p(\mathbf{z}_{t}|\mathbf{w}_{t},\mathbf{M})p(\mathbf{x}_{t}|\mathbf{z}_{t}) 
\end{align*} 
and an inference model of the form $q(\mathbf{Z}, \mathbf{W}, \mathbf{M}) = q(\mathbf{Z})q(\mathbf{W})q(\mathbf{M})$. Assume that $q(\mathbf{Z}) = \prod_{t=1}^{T}q(\mathbf{z}_{t})$ and that each $q(\mathbf{z}_{t})$ is a multivariate Gaussian whose variational parameters are supplied by a recognition model. 

Suppose that
\begin{align*}
p(\mathbf{M}) &= \mathcal{MN}_{K \times C}(\mathbf{M}|\mathbf{R}=\mathbf{R}_{0}, \mathbf{U}=\mathbf{U}_{0}, \mathbf{V}=\mathbf{I}_{C}) \\
p(\mathbf{w}_{t}) &= \mathcal{N}_{K}(\mathbf{w}_{t}|\boldsymbol{\mu}=\mathbf{0}_{K}, \boldsymbol{\Sigma}=\mathbf{I}_{K}) \\
p(\mathbf{z}_{t}|\mathbf{w}_{t},\mathbf{M}) &= \mathcal{N}_{C}(\mathbf{z}_{t}|\boldsymbol{\mu}=\mathbf{M}^{\top}\mathbf{w}_{t}, \boldsymbol{\Sigma}=\mathbf{I}_{C}) 
\end{align*}
Then the variational Bayesian update rules given in Appendix B 
simplify to:
\begin{align}
\boldsymbol{\mu}_{w_{t}} &\leftarrow (\mathbf{R}\mathbf{R}^{\top} + \mathbf{I}_{K} + C\mathbf{U})^{-1}\mathbf{R}\boldsymbol{\mu}_{z_{t}} \\
\boldsymbol{\Sigma}_{w_{t}} &\leftarrow (\mathbf{R}\mathbf{R}^{\top} + \mathbf{I}_{K} + C\mathbf{U})^{-1} \\
q^{(k+1)}(\mathbf{w}_{t}) &\leftarrow \mathcal{N}_{K}(\mathbf{w}_{t}|\boldsymbol{\mu}_{w_{t}}, \boldsymbol{\Sigma}_{w_{t}}) \\
\,\\
\mathbf{R} &\leftarrow \bigg{(}\mathbf{U}_{0}^{-1} + \sum_{t=1}^{T} (\boldsymbol{\mu}_{w_{t}}\boldsymbol{\mu}_{w_{t}}^{\top} + \boldsymbol{\Sigma}_{w_{t}})\bigg{)}^{-1}\bigg{(}\mathbf{U}_{0}^{-1}\mathbf{R}_{0} + \sum_{t=1}^{T}\boldsymbol{\mu}_{w_{t}}\boldsymbol{\mu}_{z_{t}}^{\top}\bigg{)} \\
\mathbf{U} &\leftarrow \bigg{(}\mathbf{U}_{0}^{-1} + \sum_{t=1}^{T} (\boldsymbol{\mu}_{w_{t}}\boldsymbol{\mu}_{w_{t}}^{\top} + \boldsymbol{\Sigma}_{w_{t}})\bigg{)}^{-1} \\
\mathbf{V} &\leftarrow \mathbf{I}_{C} \\
\, \\
q^{(k+1)}(\mathbf{M}) &\leftarrow \mathcal{MN}_{K \times C}(\mathbf{M}|\mathbf{R}, \mathbf{U}, \mathbf{V}) 
\end{align}
\end{theorem}
\begin{proof}
Suppose we have run the variational Bayesian updates $k$ times each. Per Appendix B, the variational optimum for $q(\mathbf{w}_{t})$ on the $(k+1)$-st iteration is given by
\begin{align}
q^{(k+1)}(\mathbf{w}_{t}) &\propto \exp \bigg{\{} \int q(\mathbf{z}_{t})q^{(k)}(\mathbf{M}) \ln p(\mathbf{z}_{t},\mathbf{w}_{t}|\mathbf{M}) d\mathbf{M}d\mathbf{z}_{t} \bigg{\}} \\
&= \exp \bigg{\{} \int q(\mathbf{z}_{t})q^{(k)}(\mathbf{M}) [\ln p(\mathbf{w}_{t}) + \ln p(\mathbf{z}_{t}|\mathbf{w}_{t},\mathbf{M})] d\mathbf{M}d\mathbf{z}_{t} \bigg{\}} \\
&= p(\mathbf{w}_{t}) \exp \bigg{\{} \int q(\mathbf{z}_{t})q^{(k)}(\mathbf{M}) \ln p(\mathbf{z}_{t}|\mathbf{w}_{t},\mathbf{M}) d\mathbf{M}d\mathbf{z}_{t} \bigg{\}} \\
&\propto p(\mathbf{w}_{t}) \exp \bigg{\{} \int q(\mathbf{z}_{t}) q^{(k)}(\mathbf{M}) \bigg{[} -\frac{1}{2} (\mathbf{z}_{t} - \mathbf{M}^{\top}\mathbf{w}_{t})^{\top} \mathbf{I}_{C}(\mathbf{z}_{t} - \mathbf{M}^{\top} \mathbf{w}_{t}) \bigg{]} d\mathbf{M} d\mathbf{z}_{t} \bigg{\}} 
\end{align}
The expectation of the above quadratic form, $(\mathbf{z}_{t} - \mathbf{M}^{\top}\mathbf{w}_{t})^{\top} \mathbf{I}_{C}(\mathbf{z}_{t} - \mathbf{M}^{\top} \mathbf{w}_{t})$, w.r.t. $q^{(k)}(\mathbf{M})$ is 
\begin{align}
&\, \tr(\mathbf{I}_{C} \cdot [\mathbf{I}_{C} \kron \mathbf{w}_{t}^{\top}\mathbf{U} \mathbf{w}_{t}]) + (\mathbf{z}_{t} - \mathbf{R}^{\top}\mathbf{w}_{t})^{\top}\mathbf{I}_{C}(\mathbf{z}_{t} - \mathbf{R}^{\top}\mathbf{w}_{t}) \\
&= \tr(\mathbf{I}_{C} \kron \mathbf{w}_{t}^{\top}\mathbf{U} \mathbf{w}_{t}) + (\mathbf{z}_{t} - \mathbf{R}^{\top}\mathbf{w}_{t})^{\top}(\mathbf{z}_{t} - \mathbf{R}^{\top}\mathbf{w}_{t}) \\
&= C \cdot \mathbf{w}_{t}^{\top}\mathbf{U}\mathbf{w}_{t} + \mathbf{z}_{t}^{\top}\mathbf{z}_{t} - 2\mathbf{z}_{t}^{\top}\mathbf{R}^{\top}\mathbf{w}_{t} + \mathbf{w}_{t}^{\top}\mathbf{R}\mathbf{R}^{\top}\mathbf{w}_{t} 
\end{align}
And we thus have
\begin{align}
q^{(k+1)}(\mathbf{w}_{t}) &\propto p(\mathbf{w}_{t}) \exp \bigg{\{} \int q(\mathbf{z}_{t}) \bigg{[} -\frac{1}{2} \big{(} C \cdot \mathbf{w}_{t}^{\top}\mathbf{U}\mathbf{w}_{t} + \mathbf{z}_{t}^{\top}\mathbf{z}_{t} - 2\mathbf{z}_{t}^{\top}\mathbf{R}^{\top}\mathbf{w}_{t} + \mathbf{w}_{t}^{\top}\mathbf{R}\mathbf{R}^{\top}\mathbf{w}_{t} \big{)} \bigg{]} d\mathbf{z}_{t} \bigg{\}} \\
&\propto p(\mathbf{w}_{t}) \exp \bigg{\{} \int q(\mathbf{z}_{t}) \bigg{[} -\frac{1}{2} \big{(} C \cdot \mathbf{w}_{t}^{\top}\mathbf{U}\mathbf{w}_{t}  - 2\mathbf{z}_{t}^{\top}\mathbf{R}^{\top}\mathbf{w}_{t} + \mathbf{w}_{t}^{\top}\mathbf{R}\mathbf{R}^{\top}\mathbf{w}_{t} \big{)} \bigg{]} d\mathbf{z}_{t} \bigg{\}} \\
&= p(\mathbf{w}_{t}) \exp \bigg{\{} -\frac{1}{2} \big{(} C \cdot \mathbf{w}_{t}^{\top}\mathbf{U}\mathbf{w}_{t}  - 2\boldsymbol{\mu}_{z_{t}}^{\top}\mathbf{R}^{\top}\mathbf{w}_{t} + \mathbf{w}_{t}^{\top}\mathbf{R}\mathbf{R}^{\top}\mathbf{w}_{t} \big{)} \bigg{\}} \\
&\propto \exp \bigg{\{} -\frac{1}{2} \mathbf{w}_{t}^{\top}\mathbf{I}_{K}\mathbf{w}_{t} \bigg{\}} \exp \bigg{\{} -\frac{1}{2} \big{(} C \cdot \mathbf{w}_{t}^{\top}\mathbf{U}\mathbf{w}_{t}  - 2\boldsymbol{\mu}_{z_{t}}^{\top}\mathbf{R}^{\top}\mathbf{w}_{t} + \mathbf{w}_{t}^{\top}\mathbf{R}\mathbf{R}^{\top}\mathbf{w}_{t} \big{)} \bigg{\}} \\
&= \exp \bigg{\{} -\frac{1}{2} \mathbf{w}_{t}^{\top}(\mathbf{R}\mathbf{R}^{\top} + \mathbf{I}_{K} + C\mathbf{U}) \mathbf{w}_{t} + \boldsymbol{\mu}_{z_{t}}^{\top}\mathbf{R}^{\top}\mathbf{w}_{t} \bigg{\}} 
\end{align}
By inspection, $q^{(k+1)}(\mathbf{w}_{t})$ has the form of a multivariate Gaussian $\exp\big{\{}-\frac{1}{2}\mathbf{w}_{t}^{\top}\boldsymbol{\Lambda}\mathbf{w}_{t} + \boldsymbol{\eta}^{\top}\mathbf{w}_{t} + a\big{\}}$ with canonical parameters 
\begin{align*}
\boldsymbol{\Lambda} &= (\mathbf{R}\mathbf{R}^{\top} + \mathbf{I}_{K} + C\mathbf{U}) \\
\boldsymbol{\eta} &= \mathbf{R}\boldsymbol{\mu}_{z_{t}}
\end{align*}
Converting from canonical parameters back to moment parameters, we obtain
\begin{align}
\boldsymbol{\mu}_{w_{t}} &= \boldsymbol{\Lambda}^{-1}\boldsymbol{\eta} = (\mathbf{R}\mathbf{R}^{\top} + \mathbf{I}_{K} + C\mathbf{U})^{-1}\mathbf{R}\boldsymbol{\mu}_{z_{t}} \\
\boldsymbol{\Sigma}_{w_{t}} &= \boldsymbol{\Lambda}^{-1} = (\mathbf{R}\mathbf{R}^{\top} + \mathbf{I}_{K} + C\mathbf{U})^{-1}
\end{align}

Thus, $q^{(k+1)}(\mathbf{w}_{t}) = \mathcal{N}_{K}(\mathbf{w}_{t}|\boldsymbol{\mu}_{w_{t}}, \boldsymbol{\Sigma}_{w_{t}})$ with $\boldsymbol{\mu}_{w}$ and $\boldsymbol{\Sigma}_{w}$ given by the moment parameters above.

Now for $q(\mathbf{M})$. Per Appendix B, the variational optimum for $q(\mathbf{M})$ on the $(k+1)$-st iteration is given by
\begin{align}
q^{(k+1)}(\mathbf{M}) &\propto p(\mathbf{M}) \prod_{t=1}^{T} \exp \bigg{\{} \int q(\mathbf{z}_{t})q^{(k+1)}(\mathbf{w}_{t}) \ln p(\mathbf{z}_{t},\mathbf{w}_{t}|\mathbf{M}) d\mathbf{w}_{t}d\mathbf{z}_{t} \bigg{\}} 
\end{align}

It will be useful to simplify each term in the above product for $q^{(k+1)}(\mathbf{M})$. We have 
\begin{align}
&\exp \bigg{\{} \int q(\mathbf{z}_{t})q^{(k+1)}(\mathbf{w}_{t}) \ln{p(\mathbf{z}_{t},\mathbf{w}_{t}|\mathbf{M})} d\mathbf{w}_{t}d\mathbf{z}_{t} \bigg{\}} \\
&= \exp \bigg{\{} \int q(\mathbf{z}_{t})q^{(k+1)}(\mathbf{w}_{t}) [\ln{p(\mathbf{w}_{t})} + \ln{p(\mathbf{z}_{t}|\mathbf{w}_{t},\mathbf{M})}] d\mathbf{w}_{t}d\mathbf{z}_{t} \bigg{\}} \\
&\propto \exp \bigg{\{} \int q(\mathbf{z}_{t})q^{(k+1)}(\mathbf{w}_{t}) [\ln{p(\mathbf{z}_{t}|\mathbf{w}_{t},\mathbf{M})}] d\mathbf{w}_{t}d\mathbf{z}_{t} \bigg{\}} \\
&\propto \exp \bigg{\{} \int q(\mathbf{z}_{t})q^{(k+1)}(\mathbf{w}_{t}) \bigg{[} -\frac{1}{2} (\mathbf{z}_{t} - \mathbf{M}^{\top}\mathbf{w}_{t})^{\top} \mathbf{I}_{C}^{-1} (\mathbf{z}_{t} - \mathbf{M}^{\top} \mathbf{w}_{t}) \bigg{]} d\mathbf{w}_{t}d\mathbf{z}_{t} \bigg{\}} \\
&= \exp \bigg{\{} \int q(\mathbf{z}_{t})q^{(k+1)}(\mathbf{w}_{t}) \bigg{[} -\frac{1}{2} \big{(} \mathbf{z}_{t}^{\top}\mathbf{z}_{t} - 2\mathbf{w}_{t}^{\top}\mathbf{M}\mathbf{z}_{t} + \mathbf{w}_{t}^{\top}\mathbf{M}\mathbf{M}^{\top}\mathbf{w}_{t} \big{)} \bigg{]} d\mathbf{w}_{t}d\mathbf{z}_{t} \bigg{\}} \\
&\propto \exp \bigg{\{} \int q(\mathbf{z}_{t})q^{(k+1)}(\mathbf{w}_{t}) \bigg{[} -\frac{1}{2} \mathbf{w}_{t}^{\top}\mathbf{M}\mathbf{M}^{\top}\mathbf{w}_{t} + \mathbf{w}_{t}^{\top}\mathbf{M}\mathbf{z}_{t} \bigg{]} d\mathbf{w}_{t}d\mathbf{z}_{t} \bigg{\}}.
\end{align}

The expectation of the above quadratic form, $\mathbf{w}_{t}^{\top}\mathbf{M}\mathbf{M}^{\top}\mathbf{w}_{t}$ w.r.t. $q^{(k+1)}(\mathbf{w}_{t})$ is
\begin{align}
&\, \tr(\mathbf{M}\mathbf{M}^{\top}\boldsymbol{\Sigma}_{w_{t}}) + \boldsymbol{\mu}_{w_{t}}^{\top}\mathbf{M}\mathbf{M}^{\top}\boldsymbol{\mu}_{w_{t}}.
\end{align}
Let $\mathbf{L}_{w_{t}}$ denote the Cholesky decomposition of $\boldsymbol{\Sigma}_{w_{t}}$. Using the cyclic invariance property of traces, the identity $\tr(A^{\top}B) = \vect(B)^{\top}\vect(A)$, and the vec trick, $\vect(AXB) = (B^{\top}\kron A)\vect(X)$, the first term equals 
\begin{align}
&\, \tr(\mathbf{M}\mathbf{M}^{\top}\boldsymbol{\Sigma}_{w_{t}}) \\
&= \tr(\mathbf{M}\mathbf{M}^{\top}\mathbf{L}_{w_{t}}\mathbf{L}_{w_{t}}^{\top}) \\
&= \tr(\mathbf{M}^{\top}\mathbf{L}_{w_{t}}\mathbf{L}_{w_{t}}^{\top}\mathbf{M}) \\
&= \tr([\mathbf{L}_{w_{t}}^{\top}\mathbf{M}]^{\top}\mathbf{L}_{w_{t}}^{\top}\mathbf{M}) \\
&= \vect(\mathbf{L}_{w_{t}}^{\top}\mathbf{M})^{\top}\vect(\mathbf{L}_{w_{t}}^{\top}\mathbf{M}) \\
&= \bigg{[}(\mathbf{I}_{C} \kron \mathbf{L}_{w_{t}}^{\top})\vect(\mathbf{M})\bigg{]}^{\top} \bigg{[}(\mathbf{I}_{C} \kron \mathbf{L}_{w_{t}}^{\top})\vect(\mathbf{M})\bigg{]} \\
&= \vect(\mathbf{M})^{\top}(\mathbf{I}_{C} \kron \mathbf{L}_{w_{t}}^{\top})^{\top} (\mathbf{I}_{C} \kron \mathbf{L}_{w_{t}}^{\top}) \vect(\mathbf{M}) \\
&= \vect(\mathbf{M})^{\top}(\mathbf{I}_{C} \kron \mathbf{L}_{w_{t}}) (\mathbf{I}_{C} \kron \mathbf{L}_{w_{t}}^{\top}) \vect(\mathbf{M}) \\
&= \vect(\mathbf{M})^{\top}(\mathbf{I}_{C} \kron \mathbf{L}_{w_{t}}\mathbf{L}_{w_{t}}^{\top}) \vect(\mathbf{M}) \\
&= \vect(\mathbf{M})^{\top}(\mathbf{I}_{C} \kron \boldsymbol{\Sigma}_{w_{t}}) \vect(\mathbf{M}).
\end{align}
Likewise, using the same properties, along with the trivial identity $\tr(a) = a$ for scalars, the second term equals
\begin{align}
&\, \boldsymbol{\mu}_{w_{t}}^{\top}\mathbf{M}\mathbf{M}^{\top}\boldsymbol{\mu}_{w_{t}} \\
&= \tr(\boldsymbol{\mu}_{w_{t}}^{\top}\mathbf{M}\mathbf{M}^{\top}\boldsymbol{\mu}_{w_{t}}) \\
&= \tr(\mathbf{M}^{\top}\boldsymbol{\mu}_{w_{t}}\boldsymbol{\mu}_{w_{t}}^{\top}\mathbf{M}) \\
&= \tr([\boldsymbol{\mu}_{w_{t}}^{\top}\mathbf{M}]^{\top}\boldsymbol{\mu}_{w_{t}}^{\top}\mathbf{M}) \\
&= \vect(\boldsymbol{\mu}_{w_{t}}^{\top}\mathbf{M})^{\top}\vect(\boldsymbol{\mu}_{w_{t}}^{\top}\mathbf{M}) \\
&= \bigg{[}(\mathbf{I}_{C} \kron \boldsymbol{\mu}_{w_{t}}^{\top})\vect(\mathbf{M})\bigg{]}^{\top} \bigg{[}(\mathbf{I}_{C} \kron \boldsymbol{\mu}_{w_{t}}^{\top})\vect(\mathbf{M})\bigg{]} \\
&= \vect(\mathbf{M})^{\top} (\mathbf{I}_{C} \kron \boldsymbol{\mu}_{w_{t}}^{\top})^{\top} (\mathbf{I}_{C} \kron \boldsymbol{\mu}_{w_{t}}^{\top}) \vect(\mathbf{M}) \\
&= \vect(\mathbf{M})^{\top} (\mathbf{I}_{C} \kron \boldsymbol{\mu}_{w_{t}}) (\mathbf{I}_{C} \kron \boldsymbol{\mu}_{w_{t}}^{\top}) \vect(\mathbf{M}) \\
&= \vect(\mathbf{M})^{\top} (\mathbf{I}_{C} \kron \boldsymbol{\mu}_{w_{t}}\boldsymbol{\mu}_{w_{t}}^{\top}) \vect(\mathbf{M}).
\end{align}
The expectation of the quadratic form $\mathbf{w}_{t}^{\top}\mathbf{M}\mathbf{M}^{\top}\mathbf{w}_{t}$ may therefore be written as
\begin{align}
&\, \vect(\mathbf{M})^{\top} (\mathbf{I}_{C} \kron \boldsymbol{\mu}_{w_{t}}\boldsymbol{\mu}_{w_{t}}^{\top}) \vect(\mathbf{M}) + \vect(\mathbf{M})^{\top} (\mathbf{I}_{C} \kron \boldsymbol{\Sigma}_{w_{t}}) \vect(\mathbf{M}) \\
&=  \vect(\mathbf{M})^{\top} (\mathbf{I}_{C} \kron (\boldsymbol{\mu}_{w_{t}}\boldsymbol{\mu}_{w_{t}}^{\top} + \boldsymbol{\Sigma}_{w_{t}})) \vect(\mathbf{M}).
\end{align}
Finally, the expectation of the other term in the integrand, $\mathbf{w}_{t}^{\top} \mathbf{M} \mathbf{z}_{t}$, w.r.t. $q^{(k+1)}(\mathbf{w}_{t})$ is 
\begin{align}
&\, \boldsymbol{\mu}_{w_{t}}^{\top} \mathbf{M} \mathbf{z}_{t} \\
&= \tr(\boldsymbol{\mu}_{w_{t}}^{\top} \mathbf{M} \mathbf{z}_{t}) \\
&= \tr( \mathbf{z}_{t} \boldsymbol{\mu}_{w_{t}}^{\top} \mathbf{M}) \\
&= \tr( [\boldsymbol{\mu}_{w_{t}} \mathbf{z}_{t}^{\top}]^{\top} \mathbf{M}) \\
&= \vect(\mathbf{M})^{\top} \vect(\boldsymbol{\mu}_{w_{t}}\mathbf{z}_{t}^{\top})  \\
&= \vect(\boldsymbol{\mu}_{w_{t}}\mathbf{z}_{t}^{\top})^{\top} \vect(\mathbf{M}), 
\end{align}
where we have used the identity $\tr(A^{\top}B) = \vect(B)^{\top}\vect(A)$ in the fifth line, and used the fact that a scalar is equal to its transpose on the sixth line.

The density $q^{(k+1)}(\mathbf{M})$ is therefore a product of $p(\mathbf{M})$ and terms of the form  
\begin{align}
&\exp \bigg{\{} \int q(\mathbf{z}_{t})q^{(k+1)}(\mathbf{w}_{t}) \ln{p(\mathbf{z}_{t},\mathbf{w}_{t}|\mathbf{M})} d\mathbf{w}_{t}d\mathbf{z}_{t} \bigg{\}} \\
&\propto \exp \bigg{\{} -\frac{1}{2} \vect(\mathbf{M})^{\top} (\mathbf{I}_{C} \kron (\boldsymbol{\mu}_{w_{t}}\boldsymbol{\mu}_{w_{t}}^{\top} + \boldsymbol{\Sigma}_{w_{t}})) \vect(\mathbf{M}) + \vect(\boldsymbol{\mu}_{w_{t}}\boldsymbol{\mu}_{z_{t}}^{\top})^{\top} \vect(\mathbf{M}) \bigg{\}} 
\end{align}
Each of these terms has the form of a multivariate Gaussian over $\vect(\mathbf{M})$, written in canonical form. Furthermore, $p(\mathbf{M})$ is a matrix-variate Gaussian with mean $\mathbf{R}_{0}$, row covariance $\mathbf{U}_{0}$ and column covariance $\mathbf{I}_{C}$. By definition \citep{GuptaNagar1999}, the density $p(\mathbf{M})$ is equal to a multivariate Gaussian density over the vectorized random variable $\vect(\mathbf{M})$,
\begin{align}
p(\mathbf{M}) = \mathcal{N}_{KC}(\vect(\mathbf{M})|\boldsymbol{\mu}=\vect(\mathbf{R}_{0}), \boldsymbol{\Sigma}=\mathbf{I}_{C}\kron\mathbf{U}_{0}) 
\end{align}

By rewriting $p(\mathbf{M})$ in canonical form, we can combine the canonical parameters of all terms in the product for $q^{(k+1)}(\mathbf{M})$ by simple addition. We have

\begin{align}
p(\mathbf{M}) &\propto \exp \bigg{\{} -\frac{1}{2} \vect(\mathbf{M})^{\top} \boldsymbol{\Lambda} \vect(\mathbf{M}) + \boldsymbol{\eta}^{\top}\vect(\mathbf{M}) \bigg{\}} \\
&= \exp \bigg{\{} -\frac{1}{2} \vect(\mathbf{M})^{\top} (\mathbf{I}_{C} \kron \mathbf{U}_{0})^{-1} \vect(\mathbf{M}) + [(\mathbf{I}_{C} \kron \mathbf{U}_{0})^{-1}\vect(\mathbf{R}_{0})]^{\top}\vect(\mathbf{M}) \bigg{\}} \\
&= \exp \bigg{\{} -\frac{1}{2} \vect(\mathbf{M})^{\top} (\mathbf{I}_{C} \kron \mathbf{U}_{0}^{-1}) \vect(\mathbf{M}) + [(\mathbf{I}_{C} \kron \mathbf{U}_{0}^{-1})\vect(\mathbf{R}_{0})]^{\top}\vect(\mathbf{M}) \bigg{\}} \\
&= \exp \bigg{\{} -\frac{1}{2} \vect(\mathbf{M})^{\top} (\mathbf{I}_{C} \kron \mathbf{U}_{0}^{-1}) \vect(\mathbf{M}) + [\vect(\mathbf{U}_{0}^{-1}\mathbf{R}_{0}\mathbf{I}_{C})]^{\top}\vect(\mathbf{M}) \bigg{\}} \\
&= \exp \bigg{\{} -\frac{1}{2} \vect(\mathbf{M})^{\top} (\mathbf{I}_{C} \kron \mathbf{U}_{0}^{-1}) \vect(\mathbf{M}) + \vect(\mathbf{U}_{0}^{-1}\mathbf{R}_{0})^{\top}\vect(\mathbf{M}) \bigg{\}} 
\end{align}

Thus, 

\begin{equation}
\begin{split}
q^{(k+1)}(\mathbf{M}) \propto &\exp \bigg{\{} -\frac{1}{2} \vect(\mathbf{M})^{\top} (\mathbf{I}_{C} \kron \mathbf{U}_{0}^{-1}) \vect(\mathbf{M}) + \vect(\mathbf{U}_{0}^{-1}\mathbf{R}_{0})^{\top}\vect(\mathbf{M}) \bigg{\}} \\
&\cdot\prod_{t=1}^{T} \exp \bigg{\{} -\frac{1}{2} \vect(\mathbf{M})^{\top} (\mathbf{I}_{C} \kron (\boldsymbol{\mu}_{w_{t}}\boldsymbol{\mu}_{w_{t}}^{\top} + \boldsymbol{\Sigma}_{w_{t}})) \vect(\mathbf{M}) + \vect(\boldsymbol{\mu}_{w_{t}}\boldsymbol{\mu}_{z_{t}}^{\top})^{\top} \vect(\mathbf{M}) \bigg{\}}
\end{split}
\end{equation}

Consequently, $q^{(k+1)}(\mathbf{M})$ simplifies to a multivariate Gaussian over $\vect(\mathbf{M})$ with canonical parameters 

\begin{align}
\boldsymbol{\Lambda} &= \mathbf{I}_{C} \kron \bigg{(} \mathbf{U}_{0}^{-1} + \sum_{t=1}^{T} (\boldsymbol{\mu}_{w_{t}}\boldsymbol{\mu}_{w_{t}}^{\top} + \boldsymbol{\Sigma}_{w_{t}}) \bigg{)} \\
\boldsymbol{\eta} &= \vect\bigg{(}\mathbf{U}_{0}^{-1}\mathbf{R}_{0} + \sum_{t=1}^{T} \boldsymbol{\mu}_{w_{t}}\boldsymbol{\mu}_{z_{t}}^{\top}\bigg{)} 
\end{align}

Consequently, the density $q^{(k+1)}(\mathbf{M})$ can be expressed as a multivariate Gaussian over $\vect(\mathbf{M})$ with moment parameters
\begin{align}
\boldsymbol{\mu} &= \boldsymbol{\Lambda}^{-1}\boldsymbol{\eta} \\
&= \bigg{[}\mathbf{I}_{C} \kron \bigg{(} \mathbf{U}_{0}^{-1} + \sum_{t=1}^{T} (\boldsymbol{\mu}_{w_{t}}\boldsymbol{\mu}_{w_{t}}^{\top} + \boldsymbol{\Sigma}_{w_{t}}) \bigg{)}\bigg{]}^{-1}\vect\bigg{(}\mathbf{U}_{0}^{-1}\mathbf{R}_{0} + \sum_{t=1}^{T} \boldsymbol{\mu}_{w_{t}}\boldsymbol{\mu}_{z_{t}}^{\top}\bigg{)} \\
&= \bigg{[}\mathbf{I}_{C} \kron \bigg{(} \mathbf{U}_{0}^{-1} + \sum_{t=1}^{T} (\boldsymbol{\mu}_{w_{t}}\boldsymbol{\mu}_{w_{t}}^{\top} + \boldsymbol{\Sigma}_{w_{t}}) \bigg{)}^{-1}\bigg{]}\vect\bigg{(}\mathbf{U}_{0}^{-1}\mathbf{R}_{0} + \sum_{t=1}^{T} \boldsymbol{\mu}_{w_{t}}\boldsymbol{\mu}_{z_{t}}^{\top}\bigg{)} \\
&= \vect\bigg{(} \bigg{(} \mathbf{U}_{0}^{-1} + \sum_{t=1}^{T} (\boldsymbol{\mu}_{w_{t}}\boldsymbol{\mu}_{w_{t}}^{\top} + \boldsymbol{\Sigma}_{w_{t}}) \bigg{)}^{-1} \bigg{(}\mathbf{U}_{0}^{-1}\mathbf{R}_{0} + \sum_{t=1}^{T} \boldsymbol{\mu}_{w_{t}}\boldsymbol{\mu}_{z_{t}}^{\top}\bigg{)} \bigg{)}, 
\end{align}
where the last line follows from the vec trick, $\vect(AXB) = (B^{\top} \kron A) \vect(X)$, and
\begin{align}
\boldsymbol{\Sigma} &= \boldsymbol{\Lambda}^{-1} \\
&= \bigg{[}\mathbf{I}_{C} \kron \bigg{(} \mathbf{U}_{0}^{-1} + \sum_{t=1}^{T} (\boldsymbol{\mu}_{w_{t}}\boldsymbol{\mu}_{w_{t}}^{\top} + \boldsymbol{\Sigma}_{w_{t}}) \bigg{)}\bigg{]}^{-1} \\
&= \mathbf{I}_{C} \kron \bigg{(} \mathbf{U}_{0}^{-1} + \sum_{t=1}^{T} (\boldsymbol{\mu}_{w_{t}}\boldsymbol{\mu}_{w_{t}}^{\top} + \boldsymbol{\Sigma}_{w_{t}}) \bigg{)}^{-1}.
\end{align}

Since $q^{(k+1)}(\mathbf{M})$ can be written as a multivariate normal distribution over $\vect(\mathbf{M})$ whose covariance factorizes via the Kronecker product, $q^{(k+1)}(\mathbf{M})$ can be written as a matrix-variate Gaussian, with mean, row covariance, and column covariance given by
\begin{align}
\mathbf{R} &= \bigg{(} \mathbf{U}_{0}^{-1} + \sum_{t=1}^{T} (\boldsymbol{\mu}_{w_{t}}\boldsymbol{\mu}_{w_{t}}^{\top} + \boldsymbol{\Sigma}_{w_{t}}) \bigg{)}^{-1} \bigg{(}\mathbf{U}_{0}^{-1}\mathbf{R}_{0} + \sum_{t=1}^{T} \boldsymbol{\mu}_{w_{t}}\boldsymbol{\mu}_{z_{t}}^{\top}\bigg{)} \\
\mathbf{U} &= \bigg{(} \mathbf{U}_{0}^{-1} + \sum_{t=1}^{T} (\boldsymbol{\mu}_{w_{t}}\boldsymbol{\mu}_{w_{t}}^{\top} + \boldsymbol{\Sigma}_{w_{t}}) \bigg{)}^{-1} \\
\mathbf{V} &= \mathbf{I}_{C}.
\end{align}

An inductive argument therefore shows that the parametric families of $q(\mathbf{w})$ and $q(\mathbf{M})$ are closed under iterations of the update equations given by the variational Bayesian EM algorithm. This concludes the derivation.
\end{proof}

This derivation was made possible by the fact that the latent space model is a type of linear Gaussian graphical model (LGGM). LGGMs also make an appearance in KFAC \citep{MartensGrosse2015}, and have been successfully used for Bayesian matrix factorization \citep{Beal2003}. Here, we use the conditional conjugacy properties of $p(\mathbf{M}|\mathbf{W},\mathbf{Z})$ to compute an approximate posterior distribution over a large latent variable $\mathbf{M}$ analytically (we found mean-field variational Bayes more robust than sampling approach of \citet{Wu2018a}). We optimize the smaller distributions $q(\mathbf{w}_{t})$ using mean-field variational Bayes as well, such that they optimize the ELBO analytically. The conditional conjugacy properties of $p(\mathbf{M}|\mathbf{W},\mathbf{Z})$ also give it a generative interpretation as a layer of model parameters, similar to LEO \citep{Rusu2019}, but without the need for an aggregate embedding to be computed beforehand, and without the need for any gradient-based optimization in latent space during inner-loop adaptation; our entire update loop is based on mean-field variational Bayes.

\subsection{Variational Bayesian Update Rules: Categorical Addresses}

For ease of exposition:
\begin{itemize}
\item Let $\Cat_{K}$ represent a probability distribution over one-hot $K$-vectors (i.e., the standard basis in $\mathbb{R}^{K}$), rather than $K$ discrete integers. 
\item Let $\DiagPart(\mathbf{A})$ denote the vector whose elements are the diagonal entries of a square matrix $\mathbf{A}$. 
\item Let $\Diag(\mathbf{v})$ denote a diagonal matrix formed from a vector $\mathbf{v}$. 
\end{itemize}

\begin{theorem}
Consider a generative model of the form 
\begin{align*}
p(\mathbf{X},\mathbf{Z},\mathbf{W},\mathbf{M}) &= p(\mathbf{M})\prod_{t=1}^{T}p(\mathbf{w}_{t})p(\mathbf{z}_{t}|\mathbf{w}_{t},\mathbf{M})p(\mathbf{x}_{t}|\mathbf{z}_{t}) 
\end{align*} 
and an inference model of the form $q(\mathbf{Z}, \mathbf{W}, \mathbf{M}) = q(\mathbf{Z})q(\mathbf{W})q(\mathbf{M})$. Assume that $q(\mathbf{Z}) = \prod_{t=1}^{T}q(\mathbf{z}_{t})$ and that each $q(\mathbf{z}_{t})$ is a multivariate Gaussian whose variational parameters are supplied by a recognition model. 
Suppose that
\begin{align*}
p(\mathbf{M}) &= \mathcal{MN}_{K \times C}(\mathbf{M}|\mathbf{R}=\mathbf{R}_{0}, \mathbf{U}=\mathbf{U}_{0}, \mathbf{V}=\mathbf{I}_{C}) \\
p(\mathbf{w}_{t}) &= \Cat_{K}(\mathbf{w}_{t}|\boldsymbol{\theta}=\frac{1}{K}\cdot\boldsymbol{1}_{K}) \\ 
p(\mathbf{z}_{t}|\mathbf{w}_{t},\mathbf{M}) &= \mathcal{N}_{C}(\mathbf{z}_{t}|\boldsymbol{\mu}=\mathbf{M}^{\top}\mathbf{w}_{t}, \boldsymbol{\Sigma}=\mathbf{I}_{C}) 
\end{align*}
Then the variational Bayesian update rules given in Appendix B 
simplify to:
\begin{align}
\boldsymbol{\theta}_{w_{t}} &\leftarrow \Softmax((-1/2)\cdot\DiagPart(\mathbf{R}\mathbf{R}^{\top} + C\mathbf{U}) + \mathbf{R}\boldsymbol{\mu}_{z_{t}}) \\
q^{(k+1)}(\mathbf{w}_{t}) &\leftarrow \Cat_{K}(\mathbf{w}_{t}|\boldsymbol{\theta}_{w_{t}}) \\
\,\\
\mathbf{R} &\leftarrow \bigg{(}\mathbf{U}_{0}^{-1} + \sum_{t=1}^{T} \Diag(\boldsymbol{\theta}_{w_{t}}) \bigg{)}^{-1}\bigg{(}\mathbf{U}_{0}^{-1}\mathbf{R}_{0} + \sum_{t=1}^{T}\boldsymbol{\theta}_{w_{t}}\boldsymbol{\mu}_{z_{t}}^{\top}\bigg{)} \\
\mathbf{U} &\leftarrow \bigg{(}\mathbf{U}_{0}^{-1} + \sum_{t=1}^{T} \Diag(\boldsymbol{\theta}_{w_{t}}) \bigg{)}^{-1} \\
\mathbf{V} &\leftarrow \mathbf{I}_{C} \\
\, \\
q^{(k+1)}(\mathbf{M}) &\leftarrow \mathcal{MN}_{K \times C}(\mathbf{M}|\mathbf{R}, \mathbf{U}, \mathbf{V}) 
\end{align}
\end{theorem}
\begin{proof}
Suppose we have run the variational Bayesian updates $k$ times each. Per Appendix B, the variational optimum for $q(\mathbf{w}_{t})$ on the $(k+1)$-st iteration is given by
\begin{align}
q^{(k+1)}(\mathbf{w}_{t}) &\propto \exp \bigg{\{} \int q(\mathbf{z}_{t})q^{(k)}(\mathbf{M}) \ln p(\mathbf{z}_{t},\mathbf{w}_{t}|\mathbf{M}) d\mathbf{M}d\mathbf{z}_{t} \bigg{\}} \\
&= \exp \bigg{\{} \int q(\mathbf{z}_{t})q^{(k)}(\mathbf{M}) [\ln p(\mathbf{w}_{t}) + \ln p(\mathbf{z}_{t}|\mathbf{w}_{t},\mathbf{M})] d\mathbf{M}d\mathbf{z}_{t} \bigg{\}} \\
&= p(\mathbf{w}_{t}) \exp \bigg{\{} \int q(\mathbf{z}_{t})q^{(k)}(\mathbf{M}) \ln p(\mathbf{z}_{t}|\mathbf{w}_{t},\mathbf{M}) d\mathbf{M}d\mathbf{z}_{t} \bigg{\}} \\
&\propto p(\mathbf{w}_{t}) \exp \bigg{\{} \int q(\mathbf{z}_{t}) q^{(k)}(\mathbf{M}) \bigg{[} -\frac{1}{2} (\mathbf{z}_{t} - \mathbf{M}^{\top}\mathbf{w}_{t})^{\top} \mathbf{I}_{C}(\mathbf{z}_{t} - \mathbf{M}^{\top} \mathbf{w}_{t}) \bigg{]} d\mathbf{M} d\mathbf{z}_{t} \bigg{\}} 
\end{align}
The expectation of the above quadratic form, $(\mathbf{z}_{t} - \mathbf{M}^{\top}\mathbf{w}_{t})^{\top} \mathbf{I}_{C}(\mathbf{z}_{t} - \mathbf{M}^{\top} \mathbf{w}_{t})$, w.r.t. $q^{(k)}(\mathbf{M})$ is 
\begin{align}
&\, \tr(\mathbf{I}_{C} \cdot [\mathbf{I}_{C} \kron \mathbf{w}_{t}^{\top}\mathbf{U} \mathbf{w}_{t}]) + (\mathbf{z}_{t} - \mathbf{R}^{\top}\mathbf{w}_{t})^{\top}\mathbf{I}_{C}(\mathbf{z}_{t} - \mathbf{R}^{\top}\mathbf{w}_{t}) \\
&= \tr(\mathbf{I}_{C} \kron \mathbf{w}_{t}^{\top}\mathbf{U} \mathbf{w}_{t}) + (\mathbf{z}_{t} - \mathbf{R}^{\top}\mathbf{w}_{t})^{\top}(\mathbf{z}_{t} - \mathbf{R}^{\top}\mathbf{w}_{t}) \\
&= C \cdot \mathbf{w}_{t}^{\top}\mathbf{U}\mathbf{w}_{t} + \mathbf{z}_{t}^{\top}\mathbf{z}_{t} - 2\mathbf{z}_{t}^{\top}\mathbf{R}^{\top}\mathbf{w}_{t} + \mathbf{w}_{t}^{\top}\mathbf{R}\mathbf{R}^{\top}\mathbf{w}_{t} 
\end{align}
Due to a uniform prior for $p(\mathbf{w}_{t})$, the value of $p(\mathbf{w}_{t})$ does not vary with $\mathbf{w}_{t}$. Thus, 
\begin{align}
q^{(k+1)}(\mathbf{w}_{t}) &\propto p(\mathbf{w}_{t}) \exp \bigg{\{} \int q(\mathbf{z}_{t}) \bigg{[} -\frac{1}{2} \big{(} C \cdot \mathbf{w}_{t}^{\top}\mathbf{U}\mathbf{w}_{t} + \mathbf{z}_{t}^{\top}\mathbf{z}_{t} - 2\mathbf{z}_{t}^{\top}\mathbf{R}^{\top}\mathbf{w}_{t} + \mathbf{w}_{t}^{\top}\mathbf{R}\mathbf{R}^{\top}\mathbf{w}_{t} \big{)} \bigg{]} d\mathbf{z}_{t} \bigg{\}} \\
&\propto p(\mathbf{w}_{t}) \exp \bigg{\{} \int q(\mathbf{z}_{t}) \bigg{[} -\frac{1}{2} \big{(} C \cdot \mathbf{w}_{t}^{\top}\mathbf{U}\mathbf{w}_{t}  - 2\mathbf{z}_{t}^{\top}\mathbf{R}^{\top}\mathbf{w}_{t} + \mathbf{w}_{t}^{\top}\mathbf{R}\mathbf{R}^{\top}\mathbf{w}_{t} \big{)} \bigg{]} d\mathbf{z}_{t} \bigg{\}} \\
&= p(\mathbf{w}_{t}) \exp \bigg{\{} -\frac{1}{2} \big{(} C \cdot \mathbf{w}_{t}^{\top}\mathbf{U}\mathbf{w}_{t}  - 2\boldsymbol{\mu}_{z_{t}}^{\top}\mathbf{R}^{\top}\mathbf{w}_{t} + \mathbf{w}_{t}^{\top}\mathbf{R}\mathbf{R}^{\top}\mathbf{w}_{t} \big{)} \bigg{\}} \\
&\propto \exp \bigg{\{} -\frac{1}{2} \big{(} C \cdot \mathbf{w}_{t}^{\top}\mathbf{U}\mathbf{w}_{t}  - 2\boldsymbol{\mu}_{z_{t}}^{\top}\mathbf{R}^{\top}\mathbf{w}_{t} + \mathbf{w}_{t}^{\top}\mathbf{R}\mathbf{R}^{\top}\mathbf{w}_{t} \big{)} \bigg{\}} \\
&= \exp \bigg{\{} -\frac{1}{2} \mathbf{w}_{t}^{\top}(\mathbf{R}\mathbf{R}^{\top} + C\mathbf{U}) \mathbf{w}_{t} + \boldsymbol{\mu}_{z_{t}}^{\top}\mathbf{R}^{\top}\mathbf{w}_{t} \bigg{\}} 
\end{align}

Evaluating the result for any choice $\mathbf{w}_{t} = \mathbf{e}_{i}$, we have:
\begin{align}
q^{(k+1)}(\mathbf{w}_{t}=\mathbf{e}_{i}) &\propto \exp \bigg{\{} -\frac{1}{2} \mathbf{e}_{i}^{\top}(\mathbf{R}\mathbf{R}^{\top} + C\mathbf{U})\mathbf{e}_{i} + \boldsymbol{\mu}_{z_{t}}^{\top}\mathbf{R}^{\top}\mathbf{e}_{i} \bigg{\}} \\
&= \exp \bigg{\{} -\frac{1}{2} [\mathbf{R}\mathbf{R}^{\top} + C\mathbf{U}]_{ii} + [\boldsymbol{\mu}_{z_{t}}^{\top}\mathbf{R}^{\top}]_{\cdot,i} \bigg{\}} \\
&= \exp \bigg{\{} -\frac{1}{2} [\mathbf{R}\mathbf{R}^{\top} + C\mathbf{U}]_{ii} + [(\mathbf{R}\boldsymbol{\mu}_{z_{t}})^\top]_{i,\cdot} \bigg{\}}
\end{align}

Thus, $q^{(k+1)}(\mathbf{w}_{t})$ has the form
\begin{align}
q^{(k+1)}(\mathbf{w}_{t}) &\propto \exp \bigg{\{} \bigg{[}-\frac{1}{2} \DiagPart(\mathbf{R}\mathbf{R}^{\top} + C\mathbf{U}) + \mathbf{R}\boldsymbol{\mu}_{z_{t}}\bigg{]}^{\top}\mathbf{w}_{t} \bigg{\}}
\end{align}
where $\DiagPart(\cdot)$ denotes the $K$-vector formed from the diagonal elements of the matrix. Thus, there are $K$ possible values and their probabilities are proportional to the exponential of each coordinate in the vector of natural parameters given above. Thus,
\begin{align}
\boldsymbol{\theta}_{w_{t}} &= \Softmax((-1/2)\cdot\DiagPart(\mathbf{R}\mathbf{R}^{\top} + C\mathbf{U}) + \mathbf{R}\boldsymbol{\mu}_{z_{t}}) \\
q^{(k+1)}(\mathbf{w}_{t}) &= \Cat_{K}(\mathbf{w}_{t}|\boldsymbol{\theta}_{w_{t}}) 
\end{align}

Now for $q(\mathbf{M})$. Per Appendix B, the variational optimum for $q(\mathbf{M})$ on the $(k+1)$-st iteration is given by
\begin{align}
q^{(k+1)}(\mathbf{M}) &\propto p(\mathbf{M}) \prod_{t=1}^{T} \exp \bigg{\{} \int q(\mathbf{z}_{t})q^{(k+1)}(\mathbf{w}_{t}) \ln p(\mathbf{z}_{t},\mathbf{w}_{t}|\mathbf{M}) d\mathbf{w}_{t}d\mathbf{z}_{t} \bigg{\}} 
\end{align}

It will be useful to simplify each term in the above product for $q^{(k+1)}(\mathbf{M})$. We have 
\begin{align}
&\exp \bigg{\{} \int q(\mathbf{z}_{t})q^{(k+1)}(\mathbf{w}_{t}) \ln{p(\mathbf{z}_{t},\mathbf{w}_{t}|\mathbf{M})} d\mathbf{w}_{t}d\mathbf{z}_{t} \bigg{\}} \\
&= \exp \bigg{\{} \int q(\mathbf{z}_{t})q^{(k+1)}(\mathbf{w}_{t}) [\ln{p(\mathbf{w}_{t})} + \ln{p(\mathbf{z}_{t}|\mathbf{w}_{t},\mathbf{M})}] d\mathbf{w}_{t}d\mathbf{z}_{t} \bigg{\}} \\
&\propto \exp \bigg{\{} \int q(\mathbf{z}_{t})q^{(k+1)}(\mathbf{w}_{t}) [\ln{p(\mathbf{z}_{t}|\mathbf{w}_{t},\mathbf{M})}] d\mathbf{w}_{t}d\mathbf{z}_{t} \bigg{\}} \\
&\propto \exp \bigg{\{} \int q(\mathbf{z}_{t})q^{(k+1)}(\mathbf{w}_{t}) \bigg{[} -\frac{1}{2} (\mathbf{z}_{t} - \mathbf{M}^{\top}\mathbf{w}_{t})^{\top} \mathbf{I}_{C}^{-1} (\mathbf{z}_{t} - \mathbf{M}^{\top} \mathbf{w}_{t}) \bigg{]} d\mathbf{w}_{t}d\mathbf{z}_{t} \bigg{\}} \\
&= \exp \bigg{\{} \int q(\mathbf{z}_{t})q^{(k+1)}(\mathbf{w}_{t}) \bigg{[} -\frac{1}{2} \big{(} \mathbf{z}_{t}^{\top}\mathbf{z}_{t} - 2\mathbf{w}_{t}^{\top}\mathbf{M}\mathbf{z}_{t} + \mathbf{w}_{t}^{\top}\mathbf{M}\mathbf{M}^{\top}\mathbf{w}_{t} \big{)} \bigg{]} d\mathbf{w}_{t}d\mathbf{z}_{t} \bigg{\}} \\
&\propto \exp \bigg{\{} \int q(\mathbf{z}_{t})q^{(k+1)}(\mathbf{w}_{t}) \bigg{[} -\frac{1}{2} \mathbf{w}_{t}^{\top}\mathbf{M}\mathbf{M}^{\top}\mathbf{w}_{t} + \mathbf{w}_{t}^{\top}\mathbf{M}\mathbf{z}_{t} \bigg{]} d\mathbf{w}_{t}d\mathbf{z}_{t} \bigg{\}} 
\end{align}

The expectation of the above quadratic form, $\mathbf{w}_{t}^{\top}\mathbf{M}\mathbf{M}^{\top}\mathbf{w}_{t}$ w.r.t. $q^{(k+1)}(\mathbf{w}_{t})$ is
\begin{align}
&\, \tr(\mathbf{M}\mathbf{M}^{\top}\boldsymbol{\Sigma}_{w_{t}}) + \boldsymbol{\mu}_{w_{t}}^{\top}\mathbf{M}\mathbf{M}^{\top}\boldsymbol{\mu}_{w_{t}} 
\end{align}
where $\mu_{w_{t}}$ and $\Sigma_{w_{t}}$ are the mean and covariance of the $K$-vectors given by the variational categorical distribution $q^{(k+1)}(\mathbf{w}_{t})$.\footnote{The formula for expectations of quadratic forms holds for expectations w.r.t. arbitrary real-valued multivariate random variables, not just Gaussian ones. See \citet{Mathai1992} for a reference.} We will simplify these symbols soon!

From here, the derivation for $q(\mathbf{M})$ follows the one from Section C.1, 
which is generic. We obtain the same expression for $q^{(k+1)}(\mathbf{M})$.

Specifically,
\begin{align}
\mathbf{R} &= \bigg{(} \mathbf{U}_{0}^{-1} + \sum_{t=1}^{T} (\boldsymbol{\mu}_{w_{t}}\boldsymbol{\mu}_{w_{t}}^{\top} + \boldsymbol{\Sigma}_{w_{t}}) \bigg{)}^{-1} \bigg{(}\mathbf{U}_{0}^{-1}\mathbf{R}_{0} + \sum_{t=1}^{T} \boldsymbol{\mu}_{w_{t}}\boldsymbol{\mu}_{z_{t}}^{\top}\bigg{)} \\
\mathbf{U} &= \bigg{(} \mathbf{U}_{0}^{-1} + \sum_{t=1}^{T} (\boldsymbol{\mu}_{w_{t}}\boldsymbol{\mu}_{w_{t}}^{\top} + \boldsymbol{\Sigma}_{w_{t}}) \bigg{)}^{-1} \\
\mathbf{V} &= \mathbf{I}_{C} \\
q^{(k+1)}(\mathbf{M}) &= \mathcal{MN}_{K \times C}(\mathbf{M}|\mathbf{R},\mathbf{U},\mathbf{V}).
\end{align}

Using an elementary covariance identity, we have
\begin{align}
\boldsymbol{\Sigma}_{w_{t}} &:= \Cov(\mathbf{w}_{t}, \mathbf{w}_{t}) \\
&=\mathbb{E}_{\mathbf{w}_{t} \sim q^{(k+1)}(\mathbf{w}_{t})} [\mathbf{w}_{t}\mathbf{w}_{t}^{\top}] - \mathbb{E}_{\mathbf{w}_{t} \sim q^{(k+1)}(\mathbf{w}_{t})}[\mathbf{w}_{t}]\mathbb{E}_{\mathbf{w}_{t} \sim q^{(k+1)}(\mathbf{w}_{t})}[\mathbf{w}_{t}]^{\top} \\
&= \mathbb{E}_{\mathbf{w}_{t} \sim q^{(k+1)}(\mathbf{w}_{t})}[\mathbf{w}_{t}\mathbf{w}_{t}^{\top}] - \boldsymbol{\mu}_{w_{t}}\boldsymbol{\mu}_{w_{t}}^{\top} \\
\end{align}
And thus, 
\begin{align}
(\boldsymbol{\mu}_{w_{t}}\boldsymbol{\mu}_{w_{t}}^{\top} + \boldsymbol{\Sigma}_{w_{t}}) &= \mathbb{E}_{\mathbf{w}_{t} \sim q^{(k+1)}(\mathbf{w}_{t})}[\mathbf{w}_{t}\mathbf{w}_{t}^{\top}] 
\end{align}
Consequently, the expectation on the right-hand side is a weighted sum of $K$ terms, corresponding to the $K$ possible values of $\mathbf{w}_{t}$. For the $i$-th term, we have $\mathbf{w}_{t} = \mathbf{e}_{i}$, and thus $\mathbf{w}_{t}\mathbf{w}_{t}^{\top}$ is a $K \times K$ matrix with a $1$ for entry $(i,i)$ and zeros elsewhere. Consequently, their probability-weighted sum is the diagonal matrix $\Diag(\boldsymbol{\theta}_{w_{t}})$. 

We thus have
\begin{align}
\mathbf{R} &= \bigg{(} \mathbf{U}_{0}^{-1} + \sum_{t=1}^{T} \Diag(\boldsymbol{\theta}_{w_{t}}) \bigg{)}^{-1} \bigg{(}\mathbf{U}_{0}^{-1}\mathbf{R}_{0} + \sum_{t=1}^{T} \boldsymbol{\theta}_{w_{t}}\boldsymbol{\mu}_{z_{t}}^{\top}\bigg{)} \\
\mathbf{U} &= \bigg{(} \mathbf{U}_{0}^{-1} + \sum_{t=1}^{T} \Diag(\boldsymbol{\theta}_{w_{t}}) \bigg{)}^{-1} \\
\mathbf{V} &= \mathbf{I}_{C} \\
q^{(k+1)}(\mathbf{M}) &= \mathcal{MN}_{K \times C}(\mathbf{M}|\mathbf{R},\mathbf{U},\mathbf{V}).
\end{align}

An inductive argument therefore shows that the parametric families of $q(\mathbf{w})$ and $q(\mathbf{M})$ are closed under iterations of the update equations given by the variational Bayesian EM algorithm. This concludes the derivation.
\end{proof}

\subsection{Variational Bayesian Update Rules: Gaussian, Mean-Shifted}

\begin{theorem}
Consider a generative model of the form 
\begin{align*}
p(\mathbf{X},\mathbf{Z},\mathbf{W},\mathbf{M},\mathbf{b}) &= p(\mathbf{M})p(\mathbf{b})\prod_{t=1}^{T}p(\mathbf{w}_{t})p(\mathbf{z}_{t}|\mathbf{w}_{t},\mathbf{M}, \mathbf{b})p(\mathbf{x}_{t}|\mathbf{z}_{t}) 
\end{align*} 
and an inference model of the form $q(\mathbf{Z}, \mathbf{W}, \mathbf{M},\mathbf{b}) = q(\mathbf{Z})q(\mathbf{W})q(\mathbf{M})q(\mathbf{b})$. Assume that $q(\mathbf{Z}) = \prod_{t=1}^{T}q(\mathbf{z}_{t})$ and that each $q(\mathbf{z}_{t})$ is a multivariate Gaussian whose variational parameters are supplied by a recognition model. 

Suppose that
\begin{align*}
p(\mathbf{M}) &= \mathcal{MN}_{K \times C}(\mathbf{M}|\mathbf{R}=\mathbf{R}_{0}, \mathbf{U}=\mathbf{U}_{0}, \mathbf{V}=\mathbf{I}_{C}) \\
p(\mathbf{w}_{t}) &= \mathcal{N}_{K}(\mathbf{w}_{t}|\boldsymbol{\mu}=\mathbf{0}_{K}, \boldsymbol{\Sigma}=\mathbf{I}_{K}) \\
p(\mathbf{z}_{t}|\mathbf{w}_{t},\mathbf{M},\mathbf{b}) &= \mathcal{N}_{C}(\mathbf{z}_{t}|\boldsymbol{\mu}=\mathbf{M}^{\top}\mathbf{w}_{t} + \mathbf{b}, \boldsymbol{\Sigma}=\sigma_{z}^{2}\mathbf{I}_{C}) \\
p(\mathbf{b}) &= \mathcal{N}_{C}(\mathbf{b}|\boldsymbol{\mu}=\boldsymbol{\mu}_{b0}, \boldsymbol{\Sigma}=\boldsymbol{\Sigma}_{b0})
\end{align*}
Then the variational Bayesian update rules are:
\begin{align}
\boldsymbol{\mu}_{b} &\leftarrow (\boldsymbol{\Sigma}_{b0}^{-1} + \sigma_{z}^{-2} T \mathbf{I}_{C})^{-1} (\boldsymbol{\Sigma}_{b0}^{-1} \boldsymbol{\mu}_{b0} + \sigma_{z}^{-2} \sum_{t=1}^{T} [\boldsymbol{\mu}_{z_{t}} - \mathbf{R}^{\top} \boldsymbol{\mu}_{w_{t}}]) \\
\boldsymbol{\Sigma}_{b} &\leftarrow (\boldsymbol{\Sigma}_{b0}^{-1} + \sigma_{z}^{-2} T \mathbf{I}_{C})^{-1} \\
q^{(k+1)}(\mathbf{b}) &\leftarrow \mathcal{N}_{C}(\mathbf{b}|\boldsymbol{\mu}=\boldsymbol{\mu}_{b}, \boldsymbol{\Sigma}=\boldsymbol{\Sigma}_{b}) \\
\, \\
\boldsymbol{\mu}_{w_{t}} &\leftarrow (\mathbf{I}_{K} + \sigma_{z}^{-2} \mathbf{R}\mathbf{R}^{\top} + \sigma_{z}^{-2} C\mathbf{U})^{-1}  \sigma_{z}^{-2} \mathbf{R}(\boldsymbol{\mu}_{z_{t}}-\boldsymbol{\mu}_{b}) \\
\boldsymbol{\Sigma}_{w_{t}} &\leftarrow (\mathbf{I}_{K} + \sigma_{z}^{-2} \mathbf{R}\mathbf{R}^{\top} + \sigma_{z}^{-2} C\mathbf{U})^{-1}  \\
q^{(k+1)}(\mathbf{w}_{t}) &\leftarrow \mathcal{N}_{K}(\mathbf{w}_{t}|\boldsymbol{\mu}_{w_{t}}, \boldsymbol{\Sigma}_{w_{t}}) \\
\,\\
\mathbf{R} &\leftarrow \bigg{(}\mathbf{U}_{0}^{-1} + \sigma_{z}^{-2} \sum_{t=1}^{T} (\boldsymbol{\mu}_{w_{t}}\boldsymbol{\mu}_{w_{t}}^{\top} + \boldsymbol{\Sigma}_{w_{t}})\bigg{)}^{-1}\bigg{(}\mathbf{U}_{0}^{-1}\mathbf{R}_{0} + \sigma_{z}^{-2} \sum_{t=1}^{T}\boldsymbol{\mu}_{w_{t}}(\boldsymbol{\mu}_{z_{t}}-\boldsymbol{\mu}_{b})^{\top}\bigg{)} \\
\mathbf{U} &\leftarrow \bigg{(}\mathbf{U}_{0}^{-1} + \sigma_{z}^{-2} \sum_{t=1}^{T} (\boldsymbol{\mu}_{w_{t}}\boldsymbol{\mu}_{w_{t}}^{\top} + \boldsymbol{\Sigma}_{w_{t}})\bigg{)}^{-1} \\
\mathbf{V} &\leftarrow \mathbf{I}_{C} \\
\, \\
q^{(k+1)}(\mathbf{M}) &\leftarrow \mathcal{MN}_{K \times C}(\mathbf{M}|\mathbf{R}, \mathbf{U}, \mathbf{V}) 
\end{align}
\end{theorem}
\begin{proof}
Suppose we have run the variational Bayesian updates $k$ times each. Then the optimal value for $q^{(k+1)}(\mathbf{b})$ is given by:
\begin{align}
q^{(k+1)}(\mathbf{b}) \propto p(\mathbf{b}) \exp \bigg{\{} \int q(\mathbf{Z})q^{(k)}(\mathbf{W})q^{(k)}(\mathbf{M}) \sum_{t=1}^{T} \ln{p(\mathbf{z}_{t}|\mathbf{w}_{t},\mathbf{b},\mathbf{M})} d\mathbf{Z}d\mathbf{W}d\mathbf{M} \bigg{\}}
\end{align}
Then, the natural parameters for $p(\mathbf{b})$ are given by $\boldsymbol{\Lambda}_{b0} = \boldsymbol{\Sigma}_{b0}^{-1}, \boldsymbol{\eta}_{b0} = \boldsymbol{\Sigma}_{b0}^{-1} \boldsymbol{\mu}_{b0}$. The exponentiated integral, meanwhile, can be rewritten as a product whose terms are indexed by $t$. The $t$-th term in the product is proportional to:
\begin{align}
&\exp \bigg{\{} \int q(\mathbf{z}_{t}) q^{(k)}(\mathbf{w}_{t}) q^{(k)}(\mathbf{M}) \cdot -\frac{1}{2} (\mathbf{z}_{t}-[\mathbf{M}^{\top}\mathbf{w}_{t}+\mathbf{b}])^{\top}(\sigma_{z}^{2}\mathbf{I}_{C})^{-1}(\mathbf{z}_{t}-[\mathbf{M}^{\top}\mathbf{w}_{t}+\mathbf{b}]) d\mathbf{w}_{t}d\mathbf{z}_{t}d\mathbf{M} \bigg{\}} \\
&\propto \exp \bigg{\{} \int q(\mathbf{z}_{t})q^{(k)}(\mathbf{w}_{t})q^{(k)}(\mathbf{M}) \cdot -\frac{1}{2} \sigma_{z}^{-2} \bigg{[} \mathbf{z}_{t}^{\top}\mathbf{z}_{t} -2\mathbf{z}_{t}^{\top}[\mathbf{M}^{\top}\mathbf{w}_{t}+\mathbf{b}] + [\mathbf{M}^{\top}\mathbf{w}_{t}+\mathbf{b}]^{\top}[\mathbf{M}^{\top}\mathbf{w}_{t}+\mathbf{b}]\bigg{]} d\mathbf{w}_{t}d\mathbf{z}_{t}d\mathbf{M} \bigg{\}} \\
&\propto \exp \bigg{\{} \int \sigma_{z}^{-2} q(\mathbf{z}_{t})q^{(k)}(\mathbf{w}_{t})q^{(k)}(\mathbf{M}) \bigg{[} \mathbf{z}_{t}^{\top} [\mathbf{M}^{\top}\mathbf{w}_{t} +\mathbf{b}] - \frac{1}{2} [\mathbf{M}^{\top}\mathbf{w}_{t}+\mathbf{b}]^{\top}[\mathbf{M}^{\top}\mathbf{w}_{t}+\mathbf{b}]\bigg{]} d\mathbf{w}_{t}d\mathbf{z}_{t}d\mathbf{M} \bigg{\}} \\
&\propto \exp \bigg{\{} \int \sigma_{z}^{-2} q(\mathbf{z}_{t})q^{(k)}(\mathbf{w}_{t})q^{(k)}(\mathbf{M}) \bigg{[} \mathbf{z}_{t}^{\top} \mathbf{b} - \frac{1}{2} \big{(} \mathbf{w}_{t}\mathbf{M}\mathbf{M}^{\top}\mathbf{w}_{t} + 2[\mathbf{M}^{\top}\mathbf{w}_{t}]^{\top}\mathbf{b} + \mathbf{b}^{\top}\mathbf{b} \big{)} \bigg{]} d\mathbf{w}_{t}d\mathbf{z}_{t}d\mathbf{M} \bigg{\}} \\
&\propto \exp \bigg{\{} \int \sigma_{z}^{-2} q(\mathbf{z}_{t})q^{(k)}(\mathbf{w}_{t})q^{(k)}(\mathbf{M}) \bigg{[} \mathbf{z}_{t}^{\top} \mathbf{b} - [\mathbf{M}^{\top}\mathbf{w}_{t}]^{\top}\mathbf{b} - \frac{1}{2} \mathbf{b}^{\top}\mathbf{b} \bigg{]} d\mathbf{w}_{t}d\mathbf{z}_{t}d\mathbf{M} \bigg{\}} \\
&= \exp \bigg{\{} \sigma_{z}^{-2} \bigg{[} - \frac{1}{2} \mathbf{b}^{\top}\mathbf{b} + [\boldsymbol{\mu}_{z_{t}} - \mathbf{R}^{\top}\boldsymbol{\mu}_{w_{t}}]^{\top}\mathbf{b} \bigg{]} \bigg{\}} 
\end{align}
where we have dropped all additive terms in the exponent which do not vary with $\mathbf{b}$. The sum of the natural parameters from the prior and above terms is then given by $\boldsymbol{\Lambda}_{b} = \boldsymbol{\Sigma}_{b0}^{-1} + \sigma_{z}^{-2} \sum_{t=1}^{T} \mathbf{I}_{C}$ and $\boldsymbol{\eta}_{b} = \boldsymbol{\Sigma}_{b0}^{-1}\boldsymbol{\mu}_{b0} + \sigma_{z}^{-2} \sum_{t=1}^{T} [\boldsymbol{\mu}_{z_{t}} - \mathbf{R}^{\top}\boldsymbol{\mu}_{w_{t}}]$. Converting back to moment parameters, we have the updates given in the theorem statement. 

The updates for $q(\mathbf{M})$ and $q(\mathbf{w}_{t})$ follow immediately from those of the first derivation we did, in Appendix C.1. 
We can group the new quantity $\mathbf{b}$ appearing in the quadratic form $(\mathbf{z}_{t}-[\mathbf{M}^{\top}\mathbf{w}_{t}+\mathbf{b}])^{\top}(\mathbf{z}_{t}-[\mathbf{M}^{\top}\mathbf{w}_{t}+\mathbf{b}])$ with the term $\mathbf{z}_{t}$ in the updates done previously, and use the mean field structure of our inference model to see that the expectations in the previous proof w.r.t. $q(\mathbf{z}_{t})$ are now replaced with expectations of the grouped term $\mathbf{z}_{t}-\mathbf{b}$ w.r.t. $q(\mathbf{z}_{t})q(\mathbf{b})$. 
\end{proof}

\subsection{Variational Bayesian Update Rules: Scalable Mixture-Based Memory Model}

\begin{theorem}
Consider a generative model of the usual form, with episode-level latent variables $\boldsymbol{\Omega} = \{ \mathbf{M}_{1:H},\mathbf{b}_{1:H} \}$ and addressing variables $\mathbf{y}_{t} = \{ \mathbf{s}_{t}, \mathbf{w}_{t} \}$. Consider an inference model of the form $q(\mathbf{Z}, \mathbf{W}, \mathbf{M}_{1:H}, \mathbf{b}_{1:H}, \mathbf{S}) = q(\mathbf{Z})q(\mathbf{S})q(\mathbf{W}|\mathbf{S})q(\mathbf{M}_{1:H})q(\mathbf{b}_{1:H})$. Assume that $q(\mathbf{Z}) = \prod_{t=1}^{T}q(\mathbf{z}_{t})$ and that each $q(\mathbf{z}_{t})$ is a multivariate Gaussian whose variational parameters are supplied by a recognition model. 

Suppose that
\begin{align}
p(\mathbf{M}_{h}) &= \mathcal{MN}_{K \times C}(\mathbf{M}|\mathbf{R}=\mathbf{R}_{0}, \mathbf{U}=\mathbf{U}_{0}, \mathbf{V}=\mathbf{I}_{C}) \\
p(\mathbf{w}_{t}|\mathbf{s}_{t}=h) &= \mathcal{N}_{K}(\mathbf{w}_{t}|\boldsymbol{\mu}=\mathbf{0}_{K}, \boldsymbol{\Sigma}=\mathbf{I}_{K}) \\
p(\mathbf{z}_{t}|\mathbf{s}_{t}=h, \mathbf{w}_{t},\mathbf{M}_{1:H},\mathbf{b}_{1:H}) &= \mathcal{N}_{C}(\mathbf{z}_{t}|\boldsymbol{\mu}=\mathbf{M}_{h}^{\top}\mathbf{w}_{t} + \mathbf{b}_{h}, \boldsymbol{\Sigma}=\sigma_{z}^{2}\mathbf{I}_{C}) \\
p(\mathbf{b}_{h}) &= \mathcal{N}_{C}(\mathbf{b}|\boldsymbol{\mu}=\boldsymbol{\mu}_{b0h}, \boldsymbol{\Sigma}=\boldsymbol{\Sigma}_{b0h}) \\
p(\mathbf{s}_{t}) &= \Cat_{H}(\mathbf{s}_{t}|\boldsymbol{\theta}=\frac{1}{H}\boldsymbol{1}_{H})
\end{align}

Then the variational Bayesian update rules are:
\begin{equation}
\begin{split}
\boldsymbol{\theta}_{s_{t}} = \Softmax \bigg{(} &-\frac{1}{2} \sigma_{z}^{-2} \bigg{[} \boldsymbol{\mu}_{w_{th}}^{\top}(\mathbf{R}_{h}\mathbf{R}_{h}^{\top}+C\mathbf{U}_{h})\boldsymbol{\mu}_{w_{th}} + \tr[(\mathbf{R}_{h}\mathbf{R}_{h}^{\top}+C\mathbf{U}_{h})\boldsymbol{\Sigma}_{w_{th}}] + \boldsymbol{\mu}_{b_{h}}^{\top} \boldsymbol{\mu}_{b_{h}} + \tr(\boldsymbol{\Sigma}_{b_{h}}) \bigg{]} \\
&+ \sigma_{z}^{-2} \boldsymbol{\mu}_{w_{th}}^{\top}\mathbf{R}_{h}(\boldsymbol{\mu}_{z_{t}} - \boldsymbol{\mu}_{b_{h}}) \\
&+ \sigma_{z}^{-2} \boldsymbol{\mu}_{b_{h}}^{\top}\boldsymbol{\mu}_{z_{t}} \\
&+ \frac{1}{2} \log{\det{\boldsymbol{\Sigma}_{w_{th}}}} \\
&- \frac{1}{2} \bigg{[} \boldsymbol{\mu}_{w_{th}}^{\top}\boldsymbol{\mu}_{w_{th}} + \tr(\boldsymbol{\Sigma}_{w_{th}}) \bigg{]} \bigg{)} 
\end{split}
\end{equation}

\begin{align}
q^{(k+1)}(\mathbf{s}_{t}) &\leftarrow \Cat_{H}(\mathbf{s}_{t}|\boldsymbol{\theta}=\boldsymbol{\theta}_{s_{t}}) \\
\, \\
\boldsymbol{\mu}_{b_{h}} &\leftarrow \bigg{(}\boldsymbol{\Sigma}_{b0}^{-1} + \sigma_{z}^{-2} \sum_{t=1}^{T} \boldsymbol{\theta}_{s_{t}}[h] \mathbf{I}_{C}\bigg{)}^{-1} \bigg{(}\boldsymbol{\Sigma}_{b0}^{-1} \boldsymbol{\mu}_{b0} + \sigma_{z}^{-2} \sum_{t=1}^{T} \boldsymbol{\theta}_{s_{t}}[h] (\boldsymbol{\mu}_{z_{t}} - \mathbf{R}_{h}^{\top} \boldsymbol{\mu}_{w_{th}})\bigg{)} \\
\boldsymbol{\Sigma}_{b_{h}} &\leftarrow \bigg{(}\boldsymbol{\Sigma}_{b0}^{-1} + \sigma_{z}^{-2} \sum_{t=1}^{T} \boldsymbol{\theta}_{s_{t}}[h] \mathbf{I}_{C}\bigg{)}^{-1} \\
q^{(k+1)}(\mathbf{b}_{h}) &\leftarrow \mathcal{N}_{C}(\mathbf{b}_{h}|\boldsymbol{\mu}=\boldsymbol{\mu}_{b_{h}}, \boldsymbol{\Sigma}=\boldsymbol{\Sigma}_{b_{h}}) \\
\, \\
\boldsymbol{\mu}_{w_{th}} &\leftarrow (\mathbf{I}_{K} + \sigma_{z}^{-2} \mathbf{R}_{h}\mathbf{R}_{h}^{\top} + \sigma_{z}^{-2} C\mathbf{U}_{h})^{-1}  \sigma_{z}^{-2} \mathbf{R}_{h}(\boldsymbol{\mu}_{z_{t}}-\boldsymbol{\mu}_{b_{h}}) \\
\boldsymbol{\Sigma}_{w_{th}} &\leftarrow (\mathbf{I}_{K} + \sigma_{z}^{-2} \mathbf{R}_{h}\mathbf{R}_{h}^{\top} + \sigma_{z}^{-2} C\mathbf{U}_{h})^{-1}  \\
q^{(k+1)}(\mathbf{w}_{t}|\mathbf{s}_{t}=h) &\leftarrow \mathcal{N}_{K}(\mathbf{w}_{t}|\boldsymbol{\mu}_{w_{th}}, \boldsymbol{\Sigma}_{w_{th}}) \\
\,\\
\mathbf{R}_{h} &\leftarrow \bigg{(}\mathbf{U}_{0}^{-1} + \sigma_{z}^{-2} \sum_{t=1}^{T} \boldsymbol{\theta}_{s_{t}}[h] (\boldsymbol{\mu}_{w_{th}}\boldsymbol{\mu}_{w_{th}}^{\top} + \boldsymbol{\Sigma}_{w_{th}})\bigg{)}^{-1}\bigg{(}\mathbf{U}_{0}^{-1}\mathbf{R}_{0} + \sigma_{z}^{-2} \sum_{t=1}^{T} \boldsymbol{\theta}_{s_{t}}[h] \boldsymbol{\mu}_{w_{th}}(\boldsymbol{\mu}_{z_{t}}-\boldsymbol{\mu}_{b_{h}})^{\top}\bigg{)} \\
\mathbf{U}_{h} &\leftarrow \bigg{(}\mathbf{U}_{0}^{-1} + \sigma_{z}^{-2} \sum_{t=1}^{T} \boldsymbol{\theta}_{s_{t}}[h] (\boldsymbol{\mu}_{w_{th}}\boldsymbol{\mu}_{w_{th}}^{\top} + \boldsymbol{\Sigma}_{w_{th}})\bigg{)}^{-1} \\
\mathbf{V}_{h} &\leftarrow \mathbf{I}_{C} \\
\, \\
q^{(k+1)}(\mathbf{M}_{h}) &\leftarrow \mathcal{MN}_{K \times C}(\mathbf{M}_{h}|\mathbf{R}_{h}, \mathbf{U}_{h}, \mathbf{V}_{h}) 
\end{align}

\end{theorem}
\begin{proof}
Suppose we have run the variational Bayesian updates $k$ times each. As per the previous derivations, the update for $q(\mathbf{s}_{t})$ can be obtained by taking the functional derivative of a Lagrangian expression $\hat{\mathcal{F}}_{X}$ formed by adding the evidence lower bound and a constraint $\lambda \bigg{[} \int q(\mathbf{s}_{t}) d\mathbf{s}_{t} - 1\bigg{]}$ to ensure $q(\mathbf{s}_{t})$ integrates to $1$. Taking the functional derivative with respect to $q(\mathbf{s}_{t})$, we have 
\begin{equation}
\begin{split}
\frac{\delta \hat{\mathcal{F}}_{X}}{\delta q(\mathbf{s}_{t})} &= \int q(\mathbf{w}_{t}|\mathbf{s}_{t}) q(\mathbf{z}_{t}) q(\mathbf{M}_{1:H})q(\mathbf{b}_{1:H}) \\
&\qquad \cdot \bigg{[}\ln{p(\mathbf{s}_{t})p(\mathbf{w}_{t}|\mathbf{s}_{t})p(\mathbf{z}_{t}|\mathbf{s}_{t},\mathbf{w}_{t},\mathbf{M}_{1:H},\mathbf{b}_{1:H})} - \ln{q(\mathbf{w}_{t}|\mathbf{s}_{t})} \bigg{]} \\ 
&\qquad d\mathbf{z}_{t}d\mathbf{w}_{t}d\mathbf{M}_{1:H}d\mathbf{b}_{1:H} \\
&-[\ln{q(\mathbf{s}_{t})}+1] + \lambda 
\end{split}
\end{equation}
Setting to zero and rearranging, we have
\begin{equation}
\begin{split}
q(\mathbf{s}_{t}) &\propto p(\mathbf{s}_{t}) \exp \bigg{\{} \int q(\mathbf{z}_{t}) q(\mathbf{w}_{t}|\mathbf{s}_{t}) q(\mathbf{M}_{1:H}) q(\mathbf{b}_{1:H}) \ln{p(\mathbf{z}_{t}|\mathbf{s}_{t},\mathbf{w}_{t},\mathbf{M}_{1:H},\mathbf{b}_{1:H})} d\mathbf{z}_{t}d\mathbf{w}_{t}d\mathbf{M}_{1:H}d\mathbf{b}_{1:H} \\
&\qquad -D_{\text{KL}}(q(\mathbf{w}_{t}|\mathbf{s}_{t})||p(\mathbf{w}_{t}|\mathbf{s}_{t})) \bigg{\}} 
\end{split}
\end{equation}
Simplifying the result using the formula for the expectation of quadratic forms, and the fact that we have a uniform prior over the variable $\mathbf{s}_{t}$, gives the update in the theorem statement. 

For $q(\mathbf{b}_{h})$ and $q(\mathbf{M}_{h})$, the variational posteriors have natural parameters obtained by summing over timesteps, as in the previous derivations. Unlike the previous derivations, the updated natural parameters contributed by the exponentiated expectations over $\ln{p(\mathbf{z}_{t}|\mathbf{s}_{t},\mathbf{w}_{t},\mathbf{M}_{1:H},\mathbf{b}_{1:H})}$ are weighted by $q(\mathbf{s}_{t}=h)$ within the exponential, and hence weighted by the scalar $\mathbf{\theta}_{s_{t}}[h]$. Converting back to moment parameters for each cluster $h$ gives the updates in the theorem statement, and follows the previous derivations for the single-cluster case; see especially Appendix  C.1. 

For $q(\mathbf{w}_{t}|\mathbf{s}_{t})$, we no longer have a mean-field assumption, and instead obtain an update by optimizing each possible conditional distribution $q(\mathbf{w}_{t}|\mathbf{s}_{t}=h)$, $h =1, \dots, H$. This gives the optimal joint distribution $q(\mathbf{s}_{t}, \mathbf{w}_{t})$ when the variational marginal distribution $q(\mathbf{s}_{t})$ is held fixed. As per the previous derivations, the update for $q(\mathbf{w}_{t}|\mathbf{s}_{t}=h)$ can be obtained by taking the functional derivative of a Lagrangian expression $\hat{\mathcal{F}}_{X}$ formed by adding the evidence lower bound and a constraint $\lambda \bigg{[} \int q(\mathbf{w}_{t}|\mathbf{s}_{t}=h) d\mathbf{w}_{t} - 1\bigg{]}$ to ensure $q(\mathbf{w}_{t}|\mathbf{s}_{t}=h)$ integrates to $1$. Taking the functional derivative with respect to $q(\mathbf{w}_{t}|\mathbf{s}_{t}=h)$, we have 
\begin{equation}
\begin{split}
\frac{\delta \hat{\mathcal{F}}_{X}}{\delta q(\mathbf{w}_{t}|\mathbf{s}_{t}=h)} &= \int q(\mathbf{s}_{t}=h) q(\mathbf{z}_{t}) q(\mathbf{M}_{1:H})q(\mathbf{b}_{1:H}) \\
&\qquad \cdot \bigg{[}\ln{p(\mathbf{w}_{t}|\mathbf{s}_{t}=h)p(\mathbf{z}_{t}|\mathbf{s}_{t}=h,\mathbf{w}_{t},\mathbf{M}_{1:H},\mathbf{b}_{1:H})}  \bigg{]} \\ 
&\qquad d\mathbf{z}_{t}d\mathbf{M}_{1:H}d\mathbf{b}_{1:H} \\
&- q(\mathbf{s}_{t}=h)[\ln{q(\mathbf{w}_{t}|\mathbf{s}_{t}=h)}+1] + \lambda
\end{split}
\end{equation}
Equating to zero, and observing that the scalar $q(\mathbf{s}_{t}=h)$ is strictly greater than zero, since the distribution $q(\mathbf{s}_{t})$ is produced by a previous iteration of our algorithm and consists of a softmax, we see that all terms above can be divided through by the nonzero scalar $q(\mathbf{s}_{t}=h)$, and the result will still equal zero. Rearranging, we see that the optimal distribution for $q(\mathbf{w}_{t}|\mathbf{s}_{t}=h)$ is given by 
\begin{equation}
\begin{split}
q(\mathbf{w}_{t}|\mathbf{s}_{t}=h) \propto p(\mathbf{w}_{t}|\mathbf{s}_{t}=h) \exp \bigg{\{} \int & q(\mathbf{z}_{t}) q(\mathbf{M}_{1:H})q(\mathbf{b}_{1:H}) \\
&\qquad \cdot \ln{p(\mathbf{z}_{t}|\mathbf{s}_{t}=h,\mathbf{w}_{t},\mathbf{M}_{1:H},\mathbf{b}_{1:H})} \\ 
&\qquad d\mathbf{z}_{t}d\mathbf{M}_{1:H}d\mathbf{b}_{1:H} \bigg{\}} 
\end{split}
\end{equation}
And thus the update for $q(\mathbf{w}_{t}|\mathbf{s}_{t}=h)$ is the same as the update for $q(\mathbf{w}_{t})$ given in the previous section, but using the variational parameters for cluster $h$. This gives the update in the theorem statement, and concludes the proof. 
\end{proof}

We note that in the non-episodic setting, it is possible to show an equivalence between the model above and a certain type of `deep' Bayesian MFA. See \citet{Beal2003} for some additional background. 

In the episodic setting, the above model is better understood as implementing a type of memory-based parameter adaptation \citep{Sprechmann2018}, in which each cluster matrix $\mathbf{M}_{h}$ acts as a layer of model parameters inside a VAE, such that sampling white noise and running it through the composed layer and decoder yields samples with a likeness of those assigned to the cluster.

\subsection{Variational Bayesian Update Rules: Tree-Structured Memory Model}

\begin{theorem}
In this section, we build on the scalable mixture model from App. C.4 
and we consider a generative model of the usual form (Fig. 1). We extend the model to collections of episode-level latent variables $\boldsymbol{\Omega} = \{ \mathbf{M}_{1:H}^{1:G},\mathbf{b}_{1:H}^{1:G} \}$ and addressing variables $\mathbf{y}_{t} = \{ \mathbf{s}_{t}^{1:G}, \mathbf{w}_{t}^{1:G} \}$. 

Here, our use of the superscript of $1:G$, above a variable name, denotes a specific collection of latent variables; the collection is indexed by $g = 1, \ldots, G$. For notational brevity, we will extend this notation to apply to collections of local latent variables as well, denoting e.g., a collection of collections, over both timesteps and the indices $g$, by the capitalized local latent variable letter followed by a superscript $1:G$. 

We consider an inference model of the form $q(\mathbf{Z}, \mathbf{S}^{1:G}, \mathbf{W}^{1:G}, \mathbf{M}_{1:H}^{1:G}, \mathbf{b}_{1:H}^{1:G}) = q(\mathbf{Z})q(\mathbf{S}^{1:G})q(\mathbf{W}^{1:G}|\mathbf{S}^{1:G})q(\mathbf{M}_{1:H}^{1:G})q(\mathbf{b}_{1:H}^{1:G})$. Assume that $q(\mathbf{Z}) = \prod_{t=1}^{T}q(\mathbf{z}_{t})$ and that each $q(\mathbf{z}_{t})$ is a multivariate Gaussian \emph{with diagonal covariance}, whose variational parameters are supplied by a recognition model. 

Let each code be a vector of dimension $C$, and let the integer $G$ be a divisor of $C$. We will denote the $g$-th consecutive set of coordinates of a code $\mathbf{z}_{t} \in \mathbb{R}^{C}$ as $\mathbf{z}_{t}^{(g)}$, and will refer to it as the $g$-th partition of the code. 

Suppose we have a generative model given by
\begin{align}
p(\boldsymbol{\Omega}) &= \prod_{g=1}^{G} p(\boldsymbol{\Omega}^{(g)}) \\
p(\boldsymbol{\Omega}^{(g)}) &= p(\mathbf{M}_{1:H}^{(g)})p(\mathbf{b}_{1:H}^{(g)}) \\
p(\mathbf{y}_{t}|\boldsymbol{\Omega}) &= p(\mathbf{y}_{t}) \\
p(\mathbf{y}_{t}) &= \prod_{g=1}^{G} p(\mathbf{y}_{t}^{(g)}) \\
p(\mathbf{y}_{t}^{(g)}) &= p(\mathbf{s}_{t}^{(g)})p(\mathbf{w}_{t}^{(g)}|\mathbf{s}_{t}^{(g)}) \\
p(\mathbf{z}_{t}|\mathbf{y}_{t}, \boldsymbol{\Omega}) &= \prod_{g=1}^{G} p(\mathbf{z}_{t}^{(g)}|\mathbf{y}_{t}^{(g)}, \boldsymbol{\Omega}^{(g)}) \\
\end{align}

In particular, let
\begin{align}
p(\mathbf{M}_{1:H}^{(g)}) &= \prod_{h=1}^{H} p(\mathbf{M}_{h}^{(g)}) \\
p(\mathbf{b}_{1:H}^{(g)}) &= \prod_{h=1}^{H} p(\mathbf{b}_{h}^{(g)}) \\
p(\mathbf{M}_{h}^{(g)}) &= \mathcal{MN}_{K \times \frac{C}{G}}(\mathbf{M}_{h}^{(g)}|\mathbf{R}=\mathbf{R}_{0}, \mathbf{U}=\mathbf{U}_{0}, \mathbf{V}=\mathbf{I}_{\frac{C}{G}}) \\
p(\mathbf{b}_{h}^{(g)}) &= \mathcal{N}_{\frac{C}{G}}(\mathbf{b}_{h}^{(g)}|\boldsymbol{\mu}=\mathbf{0}_{\frac{C}{G}}, \boldsymbol{\Sigma}=\mathbf{I}_{\frac{C}{G}}) \\
p(\mathbf{s}_{t}^{(g)}) &= \Cat_{H}(\mathbf{s}_{t}^{(g)}|\boldsymbol{\theta}=\frac{1}{H}\boldsymbol{1}_{H}) \\
p(\mathbf{w}_{t}^{(g)}|\mathbf{s}_{t}^{(g)}=h) &= \mathcal{N}_{K}(\mathbf{w}_{t}^{(g)}|\boldsymbol{\mu}=\mathbf{0}_{K}, \boldsymbol{\Sigma}=\mathbf{I}_{K}) \\
p(\mathbf{z}_{t}^{(g)}|\mathbf{s}_{t}^{(g)}=h, \mathbf{w}_{t}^{(g)}, \mathbf{M}_{1:H}^{(g)}, \mathbf{b}_{1:H}^{(g)}) &= \mathcal{N}_{\frac{C}{G}}(\mathbf{z}_{t}^{(g)}|\boldsymbol{\mu}=(\mathbf{M}_{h}^{(g)})^{\top}\mathbf{w}_{t}^{(g)} + \mathbf{b}_{h}^{(g)}, \boldsymbol{\Sigma}=\sigma_{z}^{2}\mathbf{I}_{\frac{C}{G}}) \\
\end{align}

Then the VB updates are the essentially the same as those of the scalable mixture-based memory model (App. C.4), 
but applied to each partitioned code $\mathbf{z}_{t}^{(g)}$ separately. This means any quantity from those updates depending on the scalar given by the code size $C$ is replaced with an update depending on the scalar $C / G$, and likewise for any matrices with side length $C$ in the old updates, the corresponding matrices in the updates here have side length $C / G$ instead. 
\end{theorem}

\begin{proof}
Follows from the definition of the model. All generative model terms in the ELBO depending on memory variables $\mathbf{M}_{1:H}^{1:G}$, $\mathbf{b}_{1:H}^{1:G}$ can be rewritten as sums over $g=1, \ldots, G$. Likewise for the addressing variables. Applying a similar simplification as those for timesteps when deriving the VB updates in App. B, 
we conclude that the optimal update for each grouping of latent variables factors over $g$:
\begin{align}
q(\mathbf{M}_{1:H}^{1:G}) &= \prod_{g=1}^{G} q(\mathbf{M}_{1:H}^{(g)}) \\
q(\mathbf{b}_{1:H}^{1:G}) &= \prod_{g=1}^{G} q(\mathbf{M}_{1:H}^{(g)}) \\
q(\mathbf{w}_{t}^{1:G}) &= \prod_{g=1}^{G} q(\mathbf{w}_{t}^{(g)}) \\
q(\mathbf{s}_{t}^{1:G}) &= \prod_{g=1}^{G} q(\mathbf{s}_{t}^{(g)}) \\
\end{align}
The only terms in the ELBO depending on one of the latent variables appearing on the RHS above, appear in expressions where the all other terms depend only on latent variables with the same index $g$. Thus, optimizing the ELBO can be performed separately for the latent variables associated with any given index $g$. By making the substitution $C \mapsto C / G$ in the statement of the previous theorem, the result follows. 
\end{proof}

After writing into memory, we can record the empirical pseudocounts of each cluster, conditioned on hard assignments for the previous clusters. In some cases, this allows us to store a distribution over hard assignments $\{ \mathbf{s}_{t}^{(g)} \}_{g=1}^{G}$ more space-efficiently than storing each of them in raw form. We can also fit a neural network $NN_{g}$ to predict $q(\mathbf{s}_{t}^{(g)})$ given the previous hard assignments $\mathbf{s}_{t}^{(1)}, \ldots \mathbf{s}_{t}^{(g-1)}$, which can be obtained heuristically via an argmax operation over each $q(\mathbf{s}_{t}^{(g^{\prime})})$ for $g^{\prime} = 1, \ldots g-1$. 

Informally speaking, if we consider the timestep $t$ as a random variable, the variational distributions over hard assignments at each partition step $g$ are conditionally independent given $t$, but once $t$ is marginalized out, they are not independent, and so we obtain an implicit tree structure over hard assignment prefixes (i.e., the marginal variational posterior over hard assignment sequences is a general autoregressive distribution). This tree-structure appears more formally in VQ-DRAW \citep{Nichol2020}. However in that model the codebook is dynamically parametrized by a neural network, but is not fast-adapting to new data, and mean-field variational Bayes is not used. Moreover, no covariance is modeled.  
\section{Computing the ELBO}

\subsection{Computing $D_{\text{KL}}(q(\mathbf{M})||p(\mathbf{M}))$}

In this section, we derive a formula for the KL divergence between the variational posterior $q(\mathbf{M})$ and the memory prior $p(\mathbf{M})$. Note firstly that both are matrix-variate Gaussian distributions:

\begin{align}
p(\mathbf{M}) &= \mathcal{MN}_{K \times C}(\mathbf{M}|\mathbf{R}=\mathbf{R}_{0}, \mathbf{U}=\mathbf{U}_{0}, \mathbf{V}=\mathbf{I}_{C}) \\
q(\mathbf{M}) &= \mathcal{MN}_{K \times C}(\mathbf{M}|\mathbf{R}=\mathbf{R}_{f}, \mathbf{U}=\mathbf{U}_{f}, \mathbf{V}=\mathbf{I}_{C})
\end{align}

Their densities are everywhere equal to those of multivariate Gaussians:

\begin{align}
p(\mathbf{M}) &= \mathcal{N}_{KC}(\vect(\mathbf{M})|\boldsymbol{\mu}=\vect(\mathbf{R}_{0}), \boldsymbol{\Sigma}=\mathbf{I}_{C} \kron \mathbf{U}_{0}) \\
q(\mathbf{M}) &= \mathcal{N}_{KC}(\vect(\mathbf{M})|\boldsymbol{\mu}=\vect(\mathbf{R}_{f}), \boldsymbol{\Sigma}=\mathbf{I}_{C} \kron \mathbf{U}_{f}) 
\end{align}

The KL divergence between two multivariate Gaussians is given by
\begin{equation}
\begin{split}
D_{\text{KL}}(\mathcal{N}_{k}(\boldsymbol{\mu}_{1}, \boldsymbol{\Sigma}_{1}) || \mathcal{N}_{k}(\boldsymbol{\mu}_{0}, \boldsymbol{\Sigma}_{0})) = \frac{1}{2} \bigg{[} &\tr(\boldsymbol{\Sigma}_{0}^{-1}\boldsymbol{\Sigma}_{1}) + (\boldsymbol{\mu}_{0}-\boldsymbol{\mu}_{1})^{\top}\boldsymbol{\Sigma}_{0}^{-1}(\boldsymbol{\mu}_{0}-\boldsymbol{\mu}_{1}) \\
&- k - \log{\frac{\det(\boldsymbol{\Sigma}_{1})}{\det(\boldsymbol{\Sigma}_{0})}} \bigg{]} 
\end{split}
\end{equation}

Writing the KL divergence $D_{\text{KL}}(q(\mathbf{M})||p(\mathbf{M}))$ in this format, we see that there are four terms. The first term is

\begin{align}
&\, \tr\bigg{(}(\mathbf{I}_{C} \kron \mathbf{U}_{0})^{-1}(\mathbf{I}_{C} \kron \mathbf{U}_{f})\bigg{)} \\
&= \tr\bigg{(}(\mathbf{I}_{C} \kron \mathbf{U}_{0}^{-1})(\mathbf{I}_{C} \kron \mathbf{U}_{f})\bigg{)} \\
&= \tr(\mathbf{I}_{C} \kron \mathbf{U}_{0}^{-1}\mathbf{U}_{f}) \\
&= C \tr(\mathbf{U}_{0}^{-1}\mathbf{U}_{f})
\end{align}

The second term is 

\begin{align}
&\, (\vect(\mathbf{R}_{0})-\vect(\mathbf{R}_{f}))^{\top}(\mathbf{I}_{C} \kron \mathbf{U}_{0})^{-1}(\vect(\mathbf{R}_{0})-\vect(\mathbf{R}_{f})) \\
&= \vect(\mathbf{R}_{0} - \mathbf{R}_{f})^{\top}(\mathbf{I}_{C} \kron \mathbf{U}_{0}^{-1})\vect(\mathbf{R}_{0} - \mathbf{R}_{f}) \\
&= \vect(\mathbf{R}_{0} - \mathbf{R}_{f})^{\top} \vect(\mathbf{U}_{0}^{-1}(\mathbf{R}_{0} - \mathbf{R}_{f})) \\
&= \tr((\mathbf{U}_{0}^{-1}(\mathbf{R}_{0} - \mathbf{R}_{f}))^{\top}(\mathbf{R}_{0} - \mathbf{R}_{f})) \\
&= \tr((\mathbf{R}_{0} - \mathbf{R}_{f})^{\top}\mathbf{U}_{0}^{-1}(\mathbf{R}_{0} - \mathbf{R}_{f})) \\
&= \tr((\mathbf{R}_{f}-\mathbf{R}_{0})^{\top}\mathbf{U}_{0}^{-1}(\mathbf{R}_{f}-\mathbf{R}_{0})) 
\end{align}

where we used the vec trick, $(B^{\top} \kron A)\vect(X) = \vect(AXB)$ on the third line and the identity $\tr(A^{\top}B) = \vect(B)^{\top}\vect(A)$ on the fourth line.

The third term is simply $-KC$. 

The fourth term is

\begin{align}
&\, -\log{\frac{\det(\mathbf{I}_{C}\kron\mathbf{U}_{f})}{\det(\mathbf{I}_{C}\kron\mathbf{U}_{0})}} \\
&= -\log{\det(\mathbf{I}_{C}\kron\mathbf{U}_{f})} + \log{\det(\mathbf{I}_{C}\kron\mathbf{U}_{0})} \\
&= -\log(\det(\mathbf{U}_{f})^{C}) + \log(\det(\mathbf{U}_{0})^{C}) \\
&= -C\log{\det(\mathbf{U}_{f})} + C\log{\det(\mathbf{U}_{0})} 
\end{align}

where we used the fact that $\mathbf{I}_{C} \kron \mathbf{U}$ is a diagonal block matrix and that the determinant of a diagonal block matrix is the product of the determinants of the blocks on the diagonal. 

Thus, 
\begin{equation}
\begin{split}
D_{\text{KL}}(q(\mathbf{M})||p(\mathbf{M})) = \frac{1}{2} \bigg{[} &C \tr(\mathbf{U}_{0}^{-1}\mathbf{U}_{f}) + \tr((\mathbf{R}_{f}-\mathbf{R}_{0})^{\top}\mathbf{U}_{0}^{-1}(\mathbf{R}_{f}-\mathbf{R}_{0})) \\
&- KC - C\log{\det(\mathbf{U}_{f})} + C\log{\det(\mathbf{U}_{0})} \bigg{]} 
\end{split}
\end{equation}

\subsection{A Reparametrization trick for $q(\mathbf{M})$}

In this section, we describe a reparametrization trick for matrix-variate Gaussians with identity column covariance, which are the type studied in this paper. 

Note that $\vect(\mathbf{M})$ is equal in distribution to 
\begin{align}
\vect(\mathbf{R}) + \chol{(\mathbf{I}_{C} \kron \mathbf{U})}\mathcal{E}_{KC},
\end{align} 
where $\mathcal{E}_{KC} \sim \mathcal{N}_{KC}(\cdot|\mathbf{0}_{KC}, \mathbf{I}_{KC})$ and $\chol(\cdot)$ denotes the Cholesky decomposition. 

Using the mixed-product property for Kronecker products, and the fact that $\mathbf{I}_{C} \kron \mathbf{U}$ is a block-diagonal matrix, one may show that $\chol(\mathbf{I}_{C} \kron \mathbf{U}) = \mathbf{I}_{C} \kron \chol(\mathbf{U})$, so that the above formula can be written as 
\begin{align}
&\, \vect(\mathbf{R}) + [\mathbf{I}_{C} \kron \chol(\mathbf{U})]\vect(\mathcal{E}_{K \times C}) \\
&= \vect(\mathbf{R} + \chol(\mathbf{U})\mathcal{E}_{K \times C})
\end{align}
using the vec trick. Reshaping to a matrix gives a reparametrization trick via the sampling process $\mathcal{E}_{K \times C} \sim \mathcal{MN}(\cdot|\mathbf{0}, \mathbf{I}_{K}, \mathbf{I}_{C}), \mathbf{M} = \mathbf{R} + \chol(\mathbf{U})\mathcal{E}_{K \times C}$.

\subsection{Unbiased estimator of $\mathbb{E}_{q(\mathbf{y}_{t})q(\boldsymbol{\Omega})}[D_{\text{KL}}(q(\mathbf{z}_{t})||p(\mathbf{z}_{t}|\mathbf{y}_{t},\boldsymbol{\Omega}))]$ for Tree-Structured Model}

The model in App. C5 
conforms to the specification from Section \ref{vbm_specification}. In the notation of Section 2, 
the expected KL divergence for perceptual codes $\mathbf{z}_{t}$ in the evidence lower bound (Eq. 5) is given by 
\begin{align}
\mathbb{E}_{q(\mathbf{y}_{t})q(\boldsymbol{\Omega})}[D_{\text{KL}}(q(\mathbf{z}_{t})||p(\mathbf{z}_{t}|\mathbf{y}_{t},\boldsymbol{\Omega}))]
\end{align}
Despite the fact that there are discrete latent variables $\mathbf{s}_{t}^{(g)}$ among the addressing variables $\mathbf{y}_{t}$, this expression can be computed efficiently without an exhaustive sum over realizations of $\mathbf{s}_{t} \sim \prod_{g=1}^{G} q(\mathbf{s}_{t}^{(g)})$. This is possible since (1) the variational distribution $q(\mathbf{z}_{t})$ is a diagonal Gaussian distribution and thus factors over coordinates and (2) the conditional prior $p(\mathbf{z}_{t}|\mathbf{y}_{t},\boldsymbol{\Omega})$ is a diagonal Gaussian and factors over coordinates. The integral defining the KL divergence can thus be rewritten as a sum of integrals:
\begin{align}
D_{\text{KL}}(q(\mathbf{z}_{t})||p(\mathbf{z}_{t}|\mathbf{y}_{t},\boldsymbol{\Omega})) &= \int q(\mathbf{z}_{t}) [\ln{q(\mathbf{z}_{t})} - \ln{p(\mathbf{z}_{t}|\mathbf{y}_{t},\boldsymbol{\Omega})}] d\mathbf{z}_{t} \\
&= \sum_{g=1}^{G} \int q(\mathbf{z}_{t}^{(g)}) [ \ln{q(\mathbf{z}_{t}^{(g)})} - \ln{p(\mathbf{z}_{t}^{(g)} | \mathbf{y}_{t},\boldsymbol{\Omega})} ]  d\mathbf{z}_{t}^{(g)} \\
\end{align}
and its expectation w.r.t. $\mathbf{y}_{t}, \boldsymbol{\Omega}$ can be taken afterwards. Since only the $g$-th term of the sum depends on $\mathbf{s}_{t}^{(g)}$, the expectation w.r.t. $\prod_{g=1}^{G}q(\mathbf{s}_{t}^{(g)})q(\mathbf{w}_{t}^{(g)}|\mathbf{s}_{t}^{(g)})$ can be computed efficiently using separate weighted sum for each of the terms indexed by $g$. 
\section{Dataset Preprocessing}

\subsection{Synthetic Data}
For the synthetic data experiments, we generate episodes from a simple linear Gaussian generative model, based on the one from \citet{Wu2018b}. In particular, we consider a generative model $p(\mathbf{M}) \prod_{t=1}^{T} p(\mathbf{w}_{t}) p(\mathbf{z}_{t}|\mathbf{w},\mathbf{M})$, with

\begin{align}
p(\mathbf{M}) &= \mathcal{MN}_{K \times C}(\mathbf{M}|\mathbf{R}_{0},\mathbf{U}_{0},\mathbf{V}_{0}), \\
p(\mathbf{w}_{t}) &= \mathcal{N}_{K}(\mathbf{w}_{t}|\boldsymbol{\mu}=\mathbf{0}_{K},\boldsymbol{\Sigma} = \mathbf{I}_{K}), \\
p(\mathbf{z}_{t}|\mathbf{w},\mathbf{M}) &= \mathcal{N}_{C}(\mathbf{z}_{t}|\boldsymbol{\mu}=\mathbf{M}^{\top}\mathbf{w}_{t}, \boldsymbol{\Sigma}=\sigma_{z}^{2}\mathbf{I}_{C}). 
\end{align}
The data is generated using ancestral sampling, where the top-level variable $\mathbf{M}$ is sampled once for an episode, and the per-timestep variables are sampled conditional on this variable. Our synthetic episode is then given by $\{ \mathbf{z}_{t} \}_{t=1}^{T}$. 

This generative model corresponds to a valid instance of the DKM \citep{Wu2018b}, a deep unsupervised model which uses a matrix-variate Gaussian prior for a matrix $\mathbf{M}$, uses a standard Gaussian prior for addressing weights $\mathbf{w}$, and deterministically maps the matrix-vector product $\mathbf{M}^{\top}\mathbf{w}$ to a distribution over the observation space (their paper, Appdx. A). In the simple variant here, the observations are $C$-dimensional vectors $\mathbf{z}$. 

Similar to the DKM authors, we initialized the memory prior's mean $\mathbf{R}_{0}$ with random Gaussian noise so that the DKM algorithm's first RLS step wouldn't get stuck due to symmetries in the prior mean. After the episode is generated, inference then proceeds using either the DKM algorithm or mean-field variational Bayes; both inference algorithms are benchmarked on the evidence lower bound for the generative model. 

We used the setting $\sigma_{z}^{2} = 1.0$ for the observation noise, which corresponds to a reconstruction term similar to the MSE. We used episode length $T=32$ in all cases; we observed similar results for longer episodes. 

\subsection{Natural Image Data}

Following \citet{Vinyals2016, Wu2018a, Wu2018b}, for training and evaluation of our fast-adapting models with neural networks, we generate episodes of observations by sampling without replacement from an ordinary dataset of \emph{non-grouped} data. This represents a worst-case scenario for compression in memory, since the data is close to i.i.d., rather than just conditionally i.i.d. The episodes are generated separately during each pass over the training data. We refer to each such pass as an `epoch', for ease of exposition. The preprocessing details for each dataset are given below. 

\subsubsection{CIFAR-10}

For CIFAR-10, we use the original $32 \times 32$ image size. We scale the pixel values to the range $[0, \frac{255}{256})$ and add i.i.d.  uniform random noise $\varepsilon \sim \mathcal{U}[0, \frac{1}{256})$ to each pixel each time an observation is drawn from the dataset. This serves to dequantize the pixels, which is currently best practice. This dequantization is applied during both training and evaluation. 

\subsubsection{CelebA}

For CelebA, we center-crop each image to size $108 \times 108$ and then resize to size $32 \times 32$. We scale the pixel values to the range $[0, \frac{255}{256})$ and add i.i.d. uniform random noise $\varepsilon \sim \mathcal{U}[0, \frac{1}{256})$ to each pixel each time an observation is drawn from the dataset. This serves to dequantize the pixels, which is currently best practice. This dequantization is applied during both training and evaluation. 

\section{Network Architectures and Hyperparameters}

\subsection{Quantitative Experiments}

In this section, we detail the network architecture for our quantitative experiments. 

\subsubsection{Basic Encoder-Decoder Architecture}

In this section, we detail the network architecture for our quantitative experiments. 
This architecture is similar to the some of those described by prior works \citep{Wayne2018a, Wu2018a}, and we found it to reasonably work well for a VAE, so we used it for all models in the initial experiments. 

The encoder uses three downsampling blocks. Each downsampling block begins with a convolution layer of kernel size $4 \times 4$, with $\text{num\_filter}$ filters, stride 2, identity activation, `valid' padding. Following the convolutional layer is a residual block without bottleneck \citep{He2016}, with kernel size $3 \times 3$, stride 1, ReLU activation, `same' padding. Following \citet{He2016}, we employ batch normalization in the residual blocks. The batch statistics are shared across every observation in a training batch. At test time, we use the accumulated statistics when evaluating performance. After the three downsampling blocks have been applied, the 4 dimensional tensor is flattened and linearly projected to a vector $e(\mathbf{x}_{t})$. For VAE, NS, and VBM, the vector has size $2C$, and for the DKM, the vector has size $C$. 

For the VAE, NS and VBM, the encoding $e(\mathbf{x}_{t})$ is processed by one or more MLPs. These MLPs parametrize variational parameters. Each MLP has two layers with ReLU nonlinearity in between. Each MLP's layers use hidden and output layer widths equal to twice the dimension of the variational parameters parametrized. A split is applied to the second layer, and half of the split is exponentiated in order to enforce positivity of the scale parameters for a diagonal Gaussian. 

The decoder uses three upsampling blocks. Before applying the first upsampling block, the conditioning information (a vector of dimension $C$) is linearly projected to a 4D tensor. Each upsampling block begins with a transpose convolution layer of kernel size $4 \times 4$, with $\text{num\_filter}$ filters, stride 2, identity activation, `same' padding. Following the transpose convolutional layer is a residual block without bottleneck \citep{He2016}, but with transpose convolutions, with kernel size $3 \times 3$, stride 1, ReLU activation, `same' padding. The output of the third block is cropped to the spatial dimension of the data. A distribution over pixels is parametrized using $1 \times 1$ convolutional layers to parametrize the distributional parameters. For binarized data, these distributional parameters are logits for a Bernoulli distribution. For continuous data, the distributional parameters are the mean and standard deviation for a Gaussian distribution. Glorot initialization is used for non-residual blocks; He initialization is used for the residual blocks. 

\subsubsection{Hyperparameters for Quantitative Experiments}

In this section, we detail the hyperparameters for our quantitative experiments. 

\begin{table}[H]
\begin{center}
\begin{small}
\begin{tabular}{lccccc}
\toprule
{\sc Setting} & {\sc VAE} & {\sc NS} & {\sc DKM} & {\sc VBM-Basic-1} & {\sc VBM-Basic-2} \\
\midrule
Batch size x Dupl. Factor & 16x1 & 16x1 & 16x1 & 16x1 & 8x2 \\
Episode length & 64 & 64 & 64 & 64 & 64 \\
Memory/context size & \text{-} & 6400 & $32 \times 200$ & $32 \times 200$ & $32 \times 200$ \\
Code size & 200 & 200 & 200 & 200 & 200 \\
Opt iters & \text{-} & \text{-} & \text{-} & 50 & 2 \\
Num filters & 32 & 32 & 32 & 32 & 32 \\
Adam lr & $10^{-3}$ & $10^{-3}$ & $10^{-3}$ & $10^{-3}$ & $10^{-3}$ \\
Batch normalization & \text{True} & \text{True} & \text{True} & \text{True} & \text{True} \\
Trainable memory/context prior & \text{-} & \text{False} & \text{False} & \text{False} & \text{False} \\
Sample memory posterior & \text{-} & \text{True} & \text{True} & \text{True} & \text{True} \\
Memory posterior initialization & \text{-} & \text{-} & \text{Prior} & \text{Random} & \text{Data-Dependent} \\
Max epochs & 100 & 100 & 100 & 100 & 100 \\
Early stopping epochs & 10 & 10 & 10 & 10 & 10 \\
\bottomrule
\end{tabular}

\caption{Hyperparameters for CIFAR-10 experiments.} 

\end{small}
\end{center}
\end{table}

\begin{table}[H]
\begin{center}
\begin{small}
\begin{tabular}{lccccc}
\toprule
{\sc Setting} & {\sc VAE} & {\sc NS} & {\sc DKM} & {\sc VBM-Basic-1} & {\sc VBM-Basic-2} \\
\midrule
Batch size x Dupl. Factor & 16x1 & 16x1 & 16x1 & 16x1 & 8x2 \\
Episode length & 64 & 64 & 64 & 64 & 64 \\
Memory/context size & \text{-} & 6400 & $32 \times 200$ & $32 \times 200$ & $32 \times 200$ \\
Code size & 200 & 200 & 200 & 200 & 200 \\
Opt iters & \text{-} & \text{-} & \text{-} & 50 & 2 \\
Num filters & 32 & 32 & 32 & 32 & 32 \\
Adam lr & $10^{-3}$ & $10^{-3}$ & $10^{-3}$ & $10^{-3}$ & $10^{-3}$ \\
Batch normalization & \text{True} & \text{True} & \text{True} & \text{True} & \text{True} \\
Trainable memory/context prior & \text{-} & \text{False} & \text{False} & \text{False} & \text{False} \\
Sample memory posterior & \text{-} & \text{True} & \text{True} & \text{True} & \text{True} \\
Memory posterior initialization & \text{-} & \text{-} & \text{Prior} & \text{Random} & \text{Data-Dependent} \\
Max epochs & 60 & 60 & 60 & 60 & 60 \\
Early stopping epochs & 10 & 10 & 10 & 10 & 10 \\
\bottomrule
\end{tabular}

\caption{Hyperparameters for CelebA experiments.} 

\end{small}
\end{center}
\end{table}

\subsection{Qualitative Experiments}

For the qualitative experiments, we use the more advanced models, described in Appendix C.4, C.5. 

\subsubsection{Improved Encoder-Decoder Architecture}

In this section, we detail the network architecture for our qualitative experiments, which we improved in order to stabilize training for the mixture-based memory models. 

For the improved encoder-decoder architecture, we replaced the transpose-convolutional residual blocks in the decoder with convolutional residual blocks, and reversed the order of the transpose convolutions and the residual blocks, so that a residual block was applied immediately following the linear projection. As noted in Section 3.2.2, we also replaced ReLU with the Swish-1 nonlinearity, and used group normalization instead of batch normalization. 

\subsubsection{Hyperparameters for Qualitative Experiments}

In this section, we detail the hyperparameters for our qualitative experiments. 

For the Scalable Mixture-Based Variational Bayesian Memory, we used a batch size of $B = 16$, episode length $T = 64$, number of clusters $H = 10$, memory rows $K = 6$ per cluster, optimization iterations $\ell = 10$ per episode, code size $C = 200$, number of filters $F = 64$ for the encoder and decoder, and Adam hyperparameter $\beta_{1} = 0.0$. Following Section 3.2.1, the k-means++ initialization \citep{Arthur2007} is used to initialize the mean of the variational posterior for each cluster location $\mathbf{b}_{h}$. To improve sample quality, we used the stochastic regularization method introduced in Section 3.2.3, with hyperparameters $\gamma = 0.50, \epsilon = 0.10, \delta = 0.20, \alpha = 8.0, \beta = 8.0$ and trained the model for 200 epochs with Adam stepsize $5e^{-4}$ using early stopping with patience of 10 epochs. Letting $L$ denote the number of leftover epochs from the first stage of training, we trained for an additional $10+L$ epochs with Adam stepsize $1e^{-4}$ and $\gamma = 0.50, \epsilon = 1e^{-6}, \delta = 2e^{-6}, \alpha = 8.0, \beta = 8.0$, again with early stopping and an early stopping patience of $10$ epochs. 

For the Tree-Structured Variational Bayesian Memory, we used a batch size of $B = 16$, used episode length $T = 64$, number of segments $G = 2$, clusters $H = 10$ per segment, memory rows $K = 6$ per cluster, optimization iterations $\ell = 10$ per episode, code size $C = 200$, number of filters $F = 64$ for the encoder and decoder, and Adam hyperparameter $\beta_{1} = 0.0$. To improve sample quality, we used the stochastic regularization method introduced in Section 3.2.3, with hyperparameters $\gamma = 0.50, \epsilon = 0.10, \delta = 0.20, \alpha = 8.0, \beta = 8.0$ and trained the model for 200 epochs with Adam stepsize $5e^{-4}$, using early stopping with patience of 10 epochs. Letting $L$ denote the number of leftover epochs from the first stage of training, we trained for an additional $10+L$ epochs with Adam stepsize $1e^{-4}$ and $\gamma = 0.50, \epsilon = 1e^{-6}, \delta = 2e^{-6}, \alpha = 8.0, \beta = 8.0$, again with early stopping and an early stopping patience of $10$ epochs.  
\section{Generating Samples}
\label{appdx:G}

\subsection{Generating Samples Directly from Memory}

For the `Generating from Memory' experiment (Sec. 4.4.1) 
the procedure is as follows:
\begin{enumerate}
\item Create a random episode $\mathbf{X}$ by sampling uniformly without replacement from the test set. 
\item Infer perceptual codes $q(\mathbf{z}_{t})$ for the observations in the episode using the recognition model. 
\item Write into memory by running mean-field variational Bayes to optimize the other latent variables, $q(\mathbf{\Omega}) \prod_{t=1}^{T} q(\mathbf{y}_{t})$. 
\item Save the variational parameters for episode-level latent variables $q(\mathbf{\Omega})$ \emph{only}, discarding everything else used to infer the memory distribution. 
\item Generate from memory by sampling $\boldsymbol{\Omega} \sim q(\boldsymbol{\Omega}), \mathbf{y} \sim p(\mathbf{y}|\boldsymbol{\Omega}), \mathbf{z} \sim p(\mathbf{z}|\mathbf{y},\boldsymbol{\Omega}), \mathbf{x} \sim p(\mathbf{x}|\mathbf{z})$. 
\end{enumerate}

This procedure only works for the scalable mixture-based model. For the tree-structured model, we retain estimates of the empirical pseudocounts over hard assignments to clusters, and generate mixture assignment samples from this autoregressive distribution, rather than from the generative model prior. 

\subsection{Generating Samples Iteratively}

For the `Iterative Reading' experiment (Sec. 4.4.2) 
the procedure is as follows:
\begin{enumerate}
\item Initialize a sample $\mathbf{x}$ by generating from memory as in the previous subsection. 
\item Infer variational parameters of $q(\mathbf{z})$ for the previous $\mathbf{x}$ generated. 
\item Use variational parameters of $q(\mathbf{z}), q(\boldsymbol{\Omega})$, to infer variational parameters of the other local latents, $q(\mathbf{y}) = q(\mathbf{s})q(\mathbf{w}|\mathbf{s})$, using structured mean-field variational Bayes. 
\item Sample memory variables and addressing variables from their variational distributions: $\boldsymbol{\Omega} \sim q(\boldsymbol{\Omega}), \mathbf{s} \sim q(\mathbf{s}), \mathbf{w} \sim q(\mathbf{w}|\mathbf{s})$.
\item Sample new perceptual codes and observations from their generative model distributions: $\mathbf{z} \sim p(\mathbf{z}|\mathbf{s}, \mathbf{w}, \mathbf{\Omega}), \mathbf{x} \sim p(\mathbf{x}|\mathbf{z})$. 
\item Save the observation $\mathbf{x}$. 
\item Go to step 2.
\end{enumerate}

This procedure works for both the scalable mixture-based model and the tree-structured model. Note that for the tree-structured model we follow the procedure in the previous subsection during step (1), and thus generate the initial mixture assignment samples from the  autoregressive distribution defined by empirical pseudocounts over hard assignments to clusters (as described in App. C.5), 
rather than from the generative model prior. 

\clearpage 
\section{Additional Samples}

\begin{figure}[h]
\includegraphics[height=10cm]{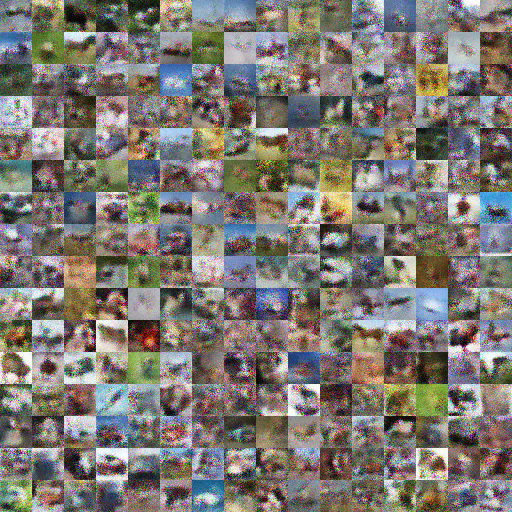} 
\caption{Samples generated directly from memory, using the scalable mixture-based memory model. We used a test-set episode of length $T  = 1280$ and a model with $H = 150$ clusters. Since the sample quality is reasonable, we believe it may be possible to scale the number of clusters sublinearly in general.} 
\end{figure} 

\end{document}